\documentclass[11pt]{article}

\PassOptionsToPackage{round}{natbib}
\usepackage[arxiv]{mystyle}

\usepackage{fullpage}
\usepackage{tikz}
\usetikzlibrary{shapes,arrows,positioning}

\newcommand{\barsigma}{\bar{\sigma}}
\newcommand{\checkA}{\check{A}}

\newcommand{\checkS}{\check{S}}
\newcommand{\checkV}{\check{V}}

\newcommand{\hatbeta}{\hat{\beta}}
\newcommand{\hatf}{\hat{f}}
\newcommand{\seq}[1]{\overline{[#1]}}
\newcommand{\sigmamin}{\sigma_{\mathrm{min}}}
\newcommand{\tildeA}{\tilde{A}}

\newcommand{\tildeS}{\tilde{S}}
\newcommand{\tildeV}{\tilde{V}}

\newcommand{\algomdp}{\text{UCRL-WVTR}}
\newcommand{\algovar}{\text{HOME}}

\title{\LARGE Horizon-Free and Instance-Dependent Regret Bounds for Reinforcement Learning with General Function Approximation}
\author{%
Jiayi Huang\thanks{Peking University. Email: \texttt{jyhuang@stu.pku.edu.cn}.}
\and
Han Zhong\thanks{Peking University. Email: \texttt{hanzhong@stu.pku.edu.cn}.}
\and
Liwei Wang\thanks{Peking University. Email: \texttt{wanglw@cis.pku.edu.cn}.}
\and
Lin F. Yang\thanks{University of California, Los Angles. Email: \texttt{linyang@ee.ucla.edu}.}
}
\begin{document}
\maketitle

\begin{abstract}
To tackle long planning horizon problems in reinforcement learning with general function approximation, we propose the first algorithm, termed as UCRL-WVTR, that achieves both \emph{horizon-free} and \emph{instance-dependent}, since it eliminates the polynomial dependency on the planning horizon. The derived regret bound is deemed \emph{sharp}, as it matches the minimax lower bound when specialized to linear mixture MDPs up to logarithmic factors. Furthermore, UCRL-WVTR is \emph{computationally efficient} with access to a regression oracle. The achievement of such a horizon-free, instance-dependent, and sharp regret bound hinges upon (i) novel algorithm designs: weighted value-targeted regression and a high-order moment estimator in the context of general function approximation; and (ii) fine-grained analyses: a novel concentration bound of weighted non-linear least squares and a refined analysis which leads to the tight instance-dependent bound. We also conduct comprehensive experiments to corroborate our theoretical findings.
\end{abstract}

\ifarxiv
\section{Introduction}
\else
\section{INTRODUCTION}
\fi
\label{sec:intro}
Reinforcement Learning (RL) plays a pivotal role in solving complex decision-making problems, where an agent interacts with the environment to learn a policy that maximizes cumulative rewards.
In the context of time-homogeneous episodic RL, where the total rewards are bounded by $1$, an open question arises concerning whether episodic RL is more difficult than bandits problems \citep{jiang2018open}.
While numerous research endeavors have been dedicated to the development of algorithms with up to poly-logarithmic dependence on the planning horizon, thus are \emph{horizon-free}, these efforts are restricted to either tabular \citep{wang2020long,zhang2021reinforcement,li2022settling,zhang2022horizon} or linear mixture Markov Decision Processes (MDPs) \citep{zhang2021improved,kim2022improved,zhou2022computationally,zhao2023variance}.
Consequently, this open question persists as a formidable challenge, necessitating further exploration and innovation to address a broader family of RL problems.

The difficulty of solving RL problems is intrinsically related to complexity of the underlying model.
For instance, the performance of an RL algorithm typically improves when the stochasticity in unknown transition kernel is diminished.
This observation motivated many works to design algorithms that achieve \emph{instance-dependent} regret bounds \citep{zanette2019tighter,zhou2023sharp,zhao2023variance}.
These bounds depend on fine-grained, problem-specific quantities, thus providing tighter guarantees than worst-case regret bounds.
Nevertheless, existing research for RL with general function approximation has primarily focused on minimizing regret in the worst-case scenario \citep{ayoub2020model,foster2023tight,agarwal2023vo}, thereby leaving an instance-dependent guarantee still challenging in this domain.

Thus, a natural question arises:
\begin{center}
    \ifarxiv
    \emph{Is efficient, horizon-free and instance-dependent learning possible\\in RL with general function approximation?}
    \else
    \emph{Is efficient, horizon-free and instance-dependent learning possible in RL with general function approximation?}
    \fi
\end{center}

In this work, we give an affirmative answer to this question by proposing an algorithm, termed as $\algomdp$, for RL with general function approximation. $\algomdp$ enjoys an $\tilde{O}(\sqrt{\dim_\mathcal{F} \log\mathcal{N}_\mathcal{F}}\cdot\sqrt{\mathcal{Q}^*} + \dim_\mathcal{F} \log\mathcal{N}_\mathcal{F})$ regret, where $\dim_\cF$ and $\log\cN_\cF$ respectively denote the generalized Eluder dimension and log-covering number of the function class $\mathcal{F}$, and $\mathcal{Q}^*$ is an instance-dependent quantity defined in Theorem~\ref{thm:mdp-regret}. The generalized Eluder dimension is an extension of Eluder dimension \citep{russo2013eluder} tailored to the context of weighted regression settings.
The derived instance-dependent regret bound exhibits a log-polynomial dependence on the planning horizon $H$. Therefore, we simultaneously achieve the best of both worlds: horizon-free and instance-dependent. For more detailed insights and explanations, we refer readers to Section~\ref{sec:mdp}.
The main theoretical contributions of this paper are summarized as follows:
\begin{itemize}
    \item Our algorithm utilizes the Value-Targeted Regression (VTR) framework proposed by \citet{ayoub2020model}, which was devised for estimating the unknown transition kernel. Our refinement of this framework includes a weighted design, which assigns higher importance to data points with low variance and uncertainty. Additionally, we adopt a high-order moment estimator to achieve more accurate estimations of variances. This improved algorithmic design plays a pivotal role in producing an accurate model estimate. Based on the estimated model, we further implement an efficient planning approach via a regression oracle for general function approximation, thus achieving computational efficiency.
    \item To theoretically characterize the deviation of the estimated model, we propose a novel Bernstein-style concentration bound for weighted non-linear regression and rigorously establish its tightness. This result provides insights for assigning weights to be both variance-aware and uncertainty-aware in our algorithm. Despite the inherent complexities posed by the non-linear function class, we have successfully conducted a refined analysis, resulting in a regret bound that simultaneously achieves horizon-free and instance-dependent.
    \item As a special case, our algorithm achieves a first-order regret scaling as $\tilde{O}(d\sqrt{V^*_1 K}+d^2)$ for linear mixture MDPs, where the transition kernel admits a linear combination of some basis transition models. This regret bound is sharp since it is reduced to $\tilde{O}(d\sqrt{K}+d^2)$ in the worst case, matching the lower bound $\Omega(d\sqrt{K})$ \citep{zhou2021nearly} when $K\ge d^2$ up to logarithmic factors. Considering this emerges as a direct consequence of a broader result, it shows that our novel algorithm designs and fine-grained analyses indeed handle the general RL problems in a sharp manner.
\end{itemize}

\paragraph{Road Map}
The rest of this paper is organized as follows. Section~\ref{sec:related} provides works that are closely related to ours. Section~\ref{sec:preliminary} introduces the formal definition of RL with general function approximation. Section~\ref{sec:mdp} states our main theoretical results. Section~\ref{sec:exp} presents the experimental results that corroborate our theoretical findings. We then make conclusions in Section~\ref{sec:conclu}. Additional experiments and detailed theoretical analyses are left in the Appendix.

\paragraph{Notations}
Let $[n]:=\{1,2,\dots,n\}$.
Let $\seq{n}:=\{0,1,\dots,n\}$.
For a collection of elements $\{X_{t}\}_{t\ge1}$, let $X_{[t]}:=\{X_1,\dots,X_t\}$.
Denote the $\epsilon$-covering number of $\cF$ w.r.t. $\ell_\infty$-norm as $\cN_\cF(\epsilon)$. 

\ifarxiv
\section{Related Work}
\else
\section{RELATED WORK}
\fi
\label{sec:related}

\paragraph{Horizon-Free Regret in RL}
\cite{jiang2018open} raised an open question regarding the comparative difficulty of RL in comparison to bandit problems. Specifically, in the context of time-homogeneous tabular RL and assuming an upper bound of $1$ on total rewards, they posited that any algorithm seeking to find an $\epsilon$-optimal policy would need to exhibit a polynomial dependence on the planning horizon $H$ in the sample complexity.
This conjecture has been challenged by a line of works that have introduced algorithms capable of achieving horizon-free regret bounds \citep{wang2020long,zhang2021reinforcement,li2022settling,zhang2022horizon}. Additionally, a series recent studies \citep{zhang2021improved,kim2022improved,zhou2022computationally,zhao2023variance} further extended horizon-free learning to linear mixture MDPs.
It is noteworthy that these aforementioned studies are based on the assumption of either a finite state space or the linear representation of transition probabilities. Regrettably, such assumptions are often less practical to align with real-world scenarios.

\paragraph{Instance-Dependent Regret in RL}
In recent years, there has been a substantial body of research dedicated to the developingment of algorithms with instance-dependent regret bounds, therefore providing tighter guarantee than traditional worst-case regret \citep{zanette2019tighter,zhou2023sharp,wagenmaker2022first,li2023variance,huang2023tackling,zhao2023variance}.
Notably, \citet{zhao2023variance} introduced a variance-adaptive algorithm with a horizon-free and variance-aware regret bound, but constrained to linear mixture MDPs.
To our best knowledge, hardly few works have explored the concept of instance-dependent regret in RL with general function approximation. The exception to this is \citet{wagenmaker2023instance}, which developed a non-asymptotic theory of instance-optimal RL with general function approximation. However, despite theoretical advancements, their algorithm falls short of practicality since it deviates significantly from minimax optimality and computational efficiency.

\paragraph{RL with General Function Approximation}
We study model-based RL with general function approximation, where the agent learns and employs an explicit model of the environment for planning and decision-making. This area has witnessed a substantial surge of research \citep{osband2014model,sun2019model,ayoub2020model, wang2020reinforcement,foster2021statistical,chen2022general,chen2022unified,zhong2022posterior,foster2023tight,wagenmaker2023instance}.
Notably, \citet{ayoub2020model} stand as the closest precursor to our work. They adopted a novel VTR approach, which evaluates models based on their ability to predict values at the next states. This VTR framework was subsequently extended and refined, culminating in the attainment of horizon-free regrets when applied to linear mixture MDPs \citep{zhou2022computationally}.
Concurrently, there exists a separate line of works concentrated on model-free RL \citep{jiang2017contextual,wang2020reinforcement,du2021bilinear,jin2021bellman,kong2021online,dann2021provably,zhong2022posterior,liu2023one,agarwal2023vo}, where the agent directly learns policies or value functions for decision-making based on their interactions with the environment.
However, it is worth noting that the algorithms mentioned above either suffer from polynomial dependence on the planning horizon $H$ or fail to provide a fine-grained instance-dependent regret guarantee.
Our work achieves the best of both worlds for the first time in RL with general function approximation.

\ifarxiv
\section{Preliminaries}
\else
\section{PRELIMINARIES}
\fi
\label{sec:preliminary}

\paragraph{MDPs with General Function Approximation}
We study time-homogeneous episodic MDPs, which can be described by a tuple $M=(\cS, \cA, H, \PP, \{r_h\}_{h\in[H]})$. Here, $\cS$ and $\cA$ are state space and action space, respectively, $H$ is the length of planning horizon, $\PP:\cS\times\cA\to\Delta(\cS)$ is the transition dynamics, $r_h:\cS\times\cA\to\RR$ is the $h$-th step deterministic reward function known to the agents\footnote{Our result can be generalized to the unknown-reward cases since learning transition dynamics is more challenging than learning rewards.}. We consider the bounded reward setting that $\sum_{h=1}^H r_h(s_h,a_h) \le 1$ for any trajectory $\{s_1,a_1,\dots,s_H,a_H\}$. We consider deterministic policy throughout this paper. A deterministic policy $\pi=\{\pi_h: \cS \to \cA\}_{h\in[H]}$ is a collection of $H$ mappings from state space to action space.
For any state-action pair $(s,a)\in\cS\times\cA$, we define the action value function $Q^\pi_h(s,a)$ and (state) value function $V^\pi_h(s)$ as follows:
\ifarxiv
\begin{align*}
    Q^\pi_h(s,a) := \EE\Big[\sum_{h'=h}^H r(s_{h'},a_{h'}) \Big| s_h=s,a_h=a\Big],\quad V^\pi_h(s) := Q^\pi_h(s,\pi_h(s)),
\end{align*}
\else
\begin{align*}
    Q^\pi_h(s,a) &:= \EE\Big[\sum_{h'=h}^H r(s_{h'},a_{h'}) \Big| s_h=s,a_h=a\Big], \\
    V^\pi_h(s) &:= Q^\pi_h(s,\pi_h(s)),
\end{align*}
\fi
where the expectation is taken with respect to the transition kernel $\PP$ and the agent's policy $\pi$. Denote $V^*_h(s):=\sup_\pi V^\pi_h(s)$ and $Q^*_h(s,a):=\sup_\pi Q^\pi_h(s,a)$ as the optimal value functions. We introduce the following shorthands for simplicity. At the $h$-th step, for any value function $V:\cS\to\RR$, denote
\ifarxiv
\begin{align*}
    [\PP V](s,a) := \EE_{s' \sim \PP(\cdot | s, a)}[V(s')],\quad [\VV V](s,a) := [\PP V^2](s,a) - [\PP V]^2(s,a)
\end{align*}
\else
\begin{align*}
    [\PP V](s,a) &:= \EE_{s' \sim \PP(\cdot | s, a)}[V(s')], \\
    [\VV V](s,a) &:= [\PP V^2](s,a) - [\PP V]^2(s,a)
\end{align*}
\fi
as the conditional expectation and variance of $V$, respectively.
We aim to design efficient algorithms for minimizing the $K$-episode regret defined as
\[
    \mathrm{Regret}(K):=\sum_{k=1}^K[V^*_1(s^k_1)-V^{\pi^k}_1(s^k_1)].
\]
To tackle problems with large state spaces, we consider MDPs with general function approximation.
\begin{assumption}[MDPs with general function approximation]
Let $\cP$ be a general function class composed of transition kernels mapping state-action pairs to measures over $\cS$. We assume the transition model $\PP$ of the MDP satisfies $\PP\in\cP$.
\end{assumption}
To characterize the complexity of the model class $\cP$, we further introduce the function class
\ifarxiv
\begin{align*}
    \cF := \{f:\cS\times\cA\times\cV\to\RR\text{ such that } \exists\PP\in\cP, f(s,a,V)=[\PP V](s,a)\text{ for any }V\in\cV\},
\end{align*}
\else
\begin{align*}
    \cF &:= \{f:\cS\times\cA\times\cV\to\RR\text{ such that } \exists\PP\in\cP, \\
    &\qquad f(s,a,V)=[\PP V](s,a)\text{ for any }V\in\cV\},
\end{align*}
\fi
where $\cV$ encompasses all value functions $V$ such that $V:\cS\to[0,1]$.
As is demonstrated by \citet{ayoub2020model}, a bijection denoted as $\phi:\mathcal{P}\rightarrow\mathcal{F}$ exists, establishing a direct correspondence between $\mathcal{P}$ and $\mathcal{F}$, such that for any $\PP\in\cP$, there exists a corresponding $f=\phi(\PP)\in\cF$ with $f(s,a,V)=[\PP V](s,a)$. For brevity, we denote $f_\PP:=\phi(\PP)$.
Additionally, we assume for any $(s,a,V)\in\cS\times\cA\times\cV$, $f_\PP(s,a,V)=\sum_{s'\in\cS}V(s')\PP(s'|s,a)$ can be efficiently evaluated within $\cO$ time. Such an assumption depends on the intrinsic property of the model class and it holds true for many MDP models. See \citet{ayoub2020model} for more discussions.
We use the covering number and generalized Eluder dimension introduced later to gauge the complexity of $\cF$.

\paragraph{Generalized Eluder Dimension}
Inspired by \citet{agarwal2023vo}, we use generalized Eluder dimension as a complexity measurement for the generalization property of the function class $\cF$. Generalized Eluder dimension is an extension of the Eluder dimension defined in \citet{russo2013eluder} to weighted regression settings. Intuitively, it can be thought of Eluder dimension of a scaling-enlarged function set.
\begin{definition}[Generalized Eluder dimension]
\label{def:general-eluder-RL}
Let $\lambda > 0$, a sequence of random elements $\bX = \{X_t\}_{t\in[T]}$ and $\bsigma = \{\sigma_t\}_{t\in[T]}$ be given. The \emph{generalized Eluder dimension} of a function class $\cF$ consisting of function $f:\cX\to\RR$ is given by
\ifarxiv
\begin{align*}
    \dim_\cF(\sigmamin,T) := \sup_{\bX,\bsigma:|\bX|=T, \bsigma\ge \sigmamin} \sum_{t=1}^T \min\rbr{1,\frac{1}{\sigma_t^2}D_\cF^2(X_t;X_{[t-1]}, \sigma_{[t-1]})},
\end{align*}
\else
\begin{align*}
    &\dim_\cF(\sigmamin,T) \\
    := &{}\sup_{\bX,\bsigma} \sum_{t=1}^T \min\rbr{1,\frac{1}{\sigma_t^2}D_\cF^2(X_t;X_{[t-1]}, \sigma_{[t-1]})}, \\
    &\text{s.t. } |\bX|=T, \bsigma\ge \sigmamin,
\end{align*}
\fi
where uncertainty $D_\cF$ is defined by
\ifarxiv
\begin{align*}
    D^2_\cF(X_t;X_{[t-1]}, \sigma_{[t-1]}) := \sup_{f_1,f_2\in\cF}\frac{\left(f_1(X_t)-f_2(X_t)\right)^2}{\sum_{s=1}^{t-1}\frac{1}{\sigma_s^2}\left(f_1(X_{s})-f_2(X_s)\right)^2+\lambda}.
\end{align*}
\else
\begin{align*}
    &D^2_\cF(X_t;X_{[t-1]}, \sigma_{[t-1]}) \\
    := &{}\sup_{f_1,f_2\in\cF}\frac{\left(f_1(X_t)-f_2(X_t)\right)^2}{\sum_{s=1}^{t-1}\frac{1}{\sigma_s^2}\left(f_1(X_{s})-f_2(X_s)\right)^2+\lambda}.
\end{align*}
\fi
We also use the notation $\dim_\cF := \dim_\cF(\sigmamin,T)$ when $\sigmamin$ and $T$ are clear from the context.
\end{definition}
\begin{remark}
We remark that the notion of generalized Eluder dimension is initially introduced by \citet{agarwal2023vo}, with a primary focus on time-inhomogeneous MDPs under model-free settings.
Furthermore, they require value closedness\footnote{The Bellman projection of any value function lies in the target function class.} \citep{wang2020reinforcement,kong2021online}. In contrast, we study time-homogenous MDPs under model-based settings, where we do not need completeness-type assumptions. More importantly, our work endeavors to pursue horizon-free learning in the context of general functions, whereas their work only obtains a near-optimal regret bound with a polynomial dependency on the planning horizon $H$.
\end{remark}

\paragraph{Linear Mixture MDPs}
We also consider linear mixture MDPs \citep{modi2020sample,jia2020model,ayoub2020model}, which is a special case where the transition kernel admits a linear representation.
\begin{assumption}[Linear mixture MDPs]
\label{ass:linear-mixture-mdp}
There exists an unknown vector $\btheta^*\in\RR^d$ and a known feature $\bphi(\cdot|\cdot,\cdot)\in\RR^d$ such that
\[
    \PP(s'|s,a)=\dotp{\bphi(s'|s,a)}{\btheta^*}
\]
for any $(s,a,s')\in\cS\times\cA\times\cS$. Meanwhile, we assume that $\|\btheta^*\|_2\le B$ and $\|\bphi_V(s,a)\|_2\le1$ for any bounded value function $V:\cS\to[0,1]$ and any state-action pair $(s,a)\in\cS\times\cA$, where
\begin{align*}
  \bphi_V(s,a):=\sum_{s'\in\cS}\bphi(s'|s,a)V(s').
\end{align*}
\end{assumption}
The generalized Eluder dimension and log-covering number of linear mixture MDPs can be simplified.
\begin{proposition}
\label{prop:linear-eluder}
\label{prop:linear-covering}
For $d$-dimensional, $B$-bounded linear mixture MDPs defined in Assumption~\ref{ass:linear-mixture-mdp}, the generalized Eluder dimension and log-covering number satisfies
\ifarxiv
\begin{align*}
    \dim_\cF(\sigmamin,T) = O(d\log(1+TB/(d\lambda\sigmamin^2)),\quad \log\cN_\cF(\epsilon) = \Theta(d\log(B/\epsilon)).
\end{align*}
\else
\begin{align*}
    \dim_\cF(\sigmamin,T) &= O(d\log(1+TB/(d\lambda\sigmamin^2)), \\
    \log\cN_\cF(\epsilon) &= \Theta(d\log(B/\epsilon)).
\end{align*}
\fi
\end{proposition}
\begin{proof}
See Appendix~\ref{proof:linear-eluder} for a detailed proof.
\end{proof}

\ifarxiv
\section{Main Results}
\else
\section{MAIN RESULTS}
\fi
\label{sec:mdp}

In this section, we propose a new algorithm $\algomdp$ for MDPs with general function approximation, as detailed in Algorithm~\ref{algo:mdp}. We first give the high-level idea, then analyze the computational complexity and regret bound.

\begin{algorithm}[t]
\caption{$\algomdp$}
\label{algo:mdp}
\begin{algorithmic}[1]
\REQUIRE $\lambda$, confidence radius $\{\hatbeta_k\}_{k \geq 1}$, level $M$, variance parameters $\sigmamin, \gamma$
\STATE For $m\in\seq{M}$, $\hatf_{1,m}\leftarrow0$, $D_{1,m},D_{1,0,m}\leftarrow\varnothing$
\FOR{$k=1,\ldots, K$}
    \STATE $V_{k,H+1}(\cdot)=0$
    \FOR{$h=H,\dots,1$}
        \STATE $Q_{k,h}(\cdot,\cdot)\leftarrow \min\{r_h(\cdot,\cdot)+\hatf_{k,0}(\cdot,\cdot,V_{k,h+1})+\hatbeta_k\cD_\cF(\cdot,\cdot,V_{k,h+1};D_{k,0}), 1\}$ \label{line:q-function}
        \STATE $V_{k,h}(\cdot)=\max_{a\in\cA}Q_{k,h}(\cdot,a)$, \\
        $\pi^k_h(\cdot)=\argmax_{a\in\cA}Q_{k,h}(\cdot,a)$
    \ENDFOR
    \STATE Receive $s^k_1$
    \FOR{$h = 1,\dots, H$}
        \STATE Take action $a_h^k  \leftarrow \pi_h^k(s_h^k)$, receive $s_{h+1}^k$
        \STATE For $m \in \seq{M}$, $z_{k,h,m} \leftarrow (s^k_h,a^k_h,V_{k,h+1}^{2^m})$, \\
        $y_{k,h,m} \leftarrow V_{k,h+1}^{2^m}(s^k_{h+1})$ \label{line:zkhm}
        \STATE $\{\bar\sigma_{k,h,m}\}_{m \in \seq{M}} \leftarrow \algovar(\{z_{k,h,m},\hatf_{k,m}$, $D_{k,h-1,m},D_{k,m}\}_{m \in \seq{M}},\hatbeta_k, \sigmamin, \gamma$) \label{line:home}
        \STATE For $m \in \seq{M}$, $D_{k,h,m} \leftarrow D_{k,h-1,m}\cup\{z_{k,h,m},\barsigma_{k,h,m}\}$
    \ENDFOR
    \STATE For $m \in \seq{M}$, $D_{k+1,m},D_{k+1,0,m} \leftarrow D_{k,H,m}$
    \STATE For $m \in \seq{M}$, $\hatf_{k+1,m} \leftarrow \argmin_{f\in\cF}$ $\sum_{i=1}^k\sum_{h=1}^H\frac{1}{\barsigma_{i,h}^2}(f(z_{i,h,m})- y_{i,h,m})^2$ \label{line:regression}
\ENDFOR
\end{algorithmic}
\end{algorithm}

\begin{algorithm}[t]
\caption{$\algovar$}
\label{algo:variance}
\begin{algorithmic}[1]
\REQUIRE $\{z_{k,h,m},\hatf_{k,m},D_{k,h-1,m},D_{k,m}\}_{m \in \seq{M}}$, $\hatbeta_k$, $\sigmamin, \gamma$
\ENSURE $\{\barsigma_{k,h,m}\}_{m \in \seq{M}}$
\FOR{$m = 0,\dots, M-1$}
    \STATE $[\bar{\VV}_{k,m}V_{k,h+1}^{2^m}](s_h^k, a_h^k) \leftarrow \hatf_{k,m+1}(z_{k,h,m+1}) - \hatf_{k,m}^2(z_{k,h,m})$
    \STATE $E_{k,h,m} \leftarrow \min\{1,2\hat\beta_k\cD_\cF(z_{k,h,m+1};D_{k,m+1})\} + \min\{1, \hat\beta_k\cD_\cF(z_{k,h,m};D_{k,m})\}$ \label{line:Ekhm}
    \STATE $\barsigma_{k,h,m}^2 \leftarrow \max\{ [\bar{\VV}_{k,m}V_{k, h+1}^{2^m}](s_h^k, a_h^k) + E_{k,h,m}, \sigmamin^2, \gamma^2\cD_\cF(z_{k,h,m};D_{k,h-1,m})\}$ \label{line:barsigma}
\ENDFOR
\STATE  $\barsigma_{k,h,M}^2 \leftarrow \max\{ 1, \sigmamin^2, \gamma^2\cD_\cF(z_{k,h,M};D_{k,h-1,M})\}$
\end{algorithmic}
\end{algorithm}

\subsection{Algorithm Description}

$\algomdp$ capitalizes on the VTR methodology, which draws inspiration from UCRL-VTR presented by \citet{ayoub2020model} on estimating transition dynamics.
Furthermore, we employ \emph{weighted} least squares regression to mitigate sub-optimality resulting from heterogeneous noise levels, specifically, differences in variances at each stage. Additionally, we utilize a high-order moment estimator for accurately estimating variances of the next-state value functions motivated by \citet{zhang2021improved,zhou2022computationally}.
We highlight the primary enhancements of $\algomdp$ in comparison to UCRL-VTR as follows:

\paragraph{Weighted VTR}
$\algomdp$ utilizes a weighted value-targeted regression framework, where the weights are chosen to take advantage of a novel concentration inequality, as articulated in Theorem~\ref{thm:wls}.
Under the guidance of this inequality, at the end of the $k$-th episode, we estimate the ground truth model $f_*=f_\PP$ as the solution of weighted least squares
{\ifarxiv\else\small\fi
\[
    \hatf_{k+1,0} = \argmin_{f\in\cF}\sum_{i=1}^k\sum_{h=1}^H\frac{1}{\barsigma_{i,h,0}^2}(f(z_{i,h,0})-V_{i,h+1}(s^i_{h+1}))^2.
\]}
Here, $z_{i,h,0}$ is defined in Line~\ref{line:zkhm} of Algorithm~\ref{algo:mdp}. And $\barsigma_{i,h,0}^2$ is an upper bound for both the conditional variance of $V_{i,h+1}(s^i_{h+1})$ and the uncertainty of $z_{i,h,0}$, defined as follows:
\ifarxiv
\begin{align*}
    \barsigma_{k,h,0}^2 = \max\{ [\bar{\VV}_{k,0}V_{k, h+1}](s_h^k, a_h^k) + E_{k,h,0}, \sigmamin^2, \gamma^2\cD_\cF(z_{k,h,0};D_{k,h-1,0})\},
\end{align*}
\else
\begin{align*}
    \barsigma_{k,h,0}^2 = \max\{& [\bar{\VV}_{k,0}V_{k, h+1}](s_h^k, a_h^k) + E_{k,h,0}, \sigmamin^2, \\
    & \gamma^2\cD_\cF(z_{k,h,0};D_{k,h-1,0})\},
\end{align*}
\fi
where $[\bar{\VV}_{k,0}V_{k, h+1}](s_h^k, a_h^k)$ is the estimate of $[\VV V_{k, h+1}](s_h^k, a_h^k)$, $E_{k,h,0}$ is the error bound specified in Line~\ref{line:Ekhm} of Algorithm~\ref{algo:variance} that satisfies
\[
    [\bar{\VV}_{k,0}V_{k, h+1}](s_h^k, a_h^k)+E_{k,h,0}\ge[\VV V_{k, h+1}](s_h^k, a_h^k)
\]
with high probability, $\sigmamin$ is a small positive constant to avoid numerical instability, $\gamma$ is a positive constant to be chosen, and $\cD_\cF(z_{k,h,0};D_{k,h-1,0})$ in Definition~\ref{def:general-eluder-RL} is the uncertainty of $z_{k,h,0}$. Then we discuss how to estimate the conditional variance of $V_{k, h+1}(s_{h+1}^k)$. Since
{\ifarxiv\else\small\fi
\[
    [\VV V_{k, h+1}](s_h^k, a_h^k) = [\PP V_{k, h+1}^2](s_h^k, a_h^k) - [\PP V_{k, h+1}]^2(s_h^k, a_h^k),
\]}
we estimate it by $[\bar{\VV}_{k,0}V_{k, h+1}](s_h^k, a_h^k)$ defined as
\[
    [\bar{\VV}_{k,0}V_{k, h+1}](s_h^k, a_h^k) = \hatf_{k,1}(z_{k,h,1}) - \hatf_{k,0}^2(z_{k,h,0}),
\]
where $z_{i,h,1}$ is defined in Line~\ref{line:zkhm} of Algorithm~\ref{algo:mdp} and $\hatf_{k,1}$ will be specified later.

\paragraph{Efficient Planning via Point-Wise Bonus Terms}
Once obtaining the estimate $\hatf_{k,0}$, we construct action value functions $\{Q_{k,h}\}_{h\in[H]}$ following the Upper Confidence Bound (UCB) scheme \citep{azar2017minimax}. More specifically, starting from step $H$ and proceeding to step $1$, $\algomdp$ estimates the upper bound of the optimal value function at each step in a point-wise manner, then takes actions optimistically:
\ifarxiv
\begin{align*}
    Q_{k,h}(s,a) = r_h(s,a)+\hatf_{k,0}(s,a,V_{k,h+1}) + \hatbeta_k\cD_\cF(s,a,V_{k,h+1};D_{k,0}),
\end{align*}
\else
\begin{align*}
    Q_{k,h}(s,a) &= r_h(s,a)+\hatf_{k,0}(s,a,V_{k,h+1}) \\
    &\qquad + \hatbeta_k\cD_\cF(s,a,V_{k,h+1};D_{k,0}),
\end{align*}
\fi
Here $\hatbeta_k=\tilde{O}(\sqrt{\log\cN_\cF})$ according to Theorem~\ref{thm:wls}. However, it is worth noting that in the realm of non-linear functions, uncertainty $\cD_\cF$ generally lacks an analytical expression, impeding the computational efficiency of our algorithm. To address such a challenge, we employ a technique to efficiently compute uncertainty $\cD_\cF$ via a regression oracle, as explained later. In contrast, \citet{ayoub2020model} construct value functions through optimistic planning, which is generally computationally intractable.

\paragraph{High-Order Moment Estimator}
We utilize the high-order moment estimator in Algorithm~\ref{algo:variance} to achieve enhanced precision in estimating the ground truth model $f_*$. As previously discussed, the estimate $\hatf_{k,0}$ necessitates weights $\barsigma_{k,h,0}$, which relies on $\hatf_{k,1}$ for estimating variance of $V_{k, h+1}$. And $\hatf_{k,1}$ can be derived by weighted least squares with predictors $z_{i,h,1}$, weights $\barsigma_{i,h,1}$ and targets $V_{i,h+1}^2(s^i_{h+1})$. Here $\barsigma_{k,h,1}$ are in a similar manner as $\barsigma_{k,h,0}$, depending on the variance of $V_{k, h+1}^2$. Recursively, we estimate variance of $V_{k,h+1}^{2^m}$, which is the conditional $2^{m+1}$-th central moment of $V_{k,h+1}$ until $m=M$, where $M$ is chosen to meet the desired level of precision. Last, variance of $V_{k,h+1}^{2^M}$ is estimated with its upper bound $1$.
It is worth noting that this higher-order moment estimator has been previously utilized in research dedicated to horizon-free learning. However, these works focused on either tabular \citep{zhang2021reinforcement,zhang2022horizon} or linear mixture MDPs \citep{zhang2021improved,zhou2022computationally}, while we study general function approximation.

\subsection{Computational Complexity}
\label{sec:computational-complexity}

Due to the non-linear nature of the function class $\cF$, closed-form solutions of weighted least squares $\hatf_{k,m}$ in Line~\ref{line:regression} and uncertainty $\cD_\cF$ in Definition~\ref{def:general-eluder-RL} are not readily attainable. To assess the computational efficiency of our proposed algorithm in such a context, we turn to the concept of oracle complexity \citep{wang2020reinforcement,kong2021online}, which measures the number of calls to some optimization oracles.

\paragraph{Regression Oracle}
We introduce the regression oracle in Assumption~\ref{ass:regression-oracle} for solving the weighted non-linear least squares regression. Inspired by \citet{kong2021online}, we leverage this oracle to compute uncertainty $\cD_\cF$ through a binary search procedure. As detailed in Appendix~\ref{sec:uncertainty-compute}, we demonstrate that this quantity can be estimated efficiently with a mere $\tilde{O}(1)$ number of calls to the regression oracle.
\begin{assumption}[Regression oracle]
\label{ass:regression-oracle}
We assume access to a weighted least squares \emph{regression oracle}, which takes a function class $\cG$ consisting of function $g:\cX\to\RR$ and a $t$-sized of weighted examples $\{(X_s,v_s,Y_s)\}_{s\in[t]}\subset\cX\times\RR^+\times\RR$ as input, and outputs the solution of weighted least squares $\hat{g}$ within $\cR$ time, where $\hat{g}$ is defined as
\[
    \hat{g} = \argmin_{g\in\cG} \sum_{s=1}^t v_s(g(X_s)-Y_s)^2.
\]
\end{assumption}
The availability of a regression oracle is a reasonably mild assumption, which appeared in many works concerning general function approximation \citep{krishnamurthy2017active,foster2018practical,kong2021online,agarwal2023vo}. It is noteworthy that this regression oracle admits an analytical solution under linear mixture models \citep{ayoub2020model}. In more general scenarios, when the function class $\cF$ is characterized as a collection of differentiable functions, such as neural networks, the implementation of the regression oracle becomes feasible and computationally efficient through the utilization of gradient-based optimization algorithms \citep{bubeck2015convex}.

\paragraph{Computational Complexity of $\algomdp$}
Recall that $\cO$ denotes the evaluation time of any function $f\in\cF$ and $\cR$ represents the computational cost of the regression oracle.
We consider the computation cost of the $k$-th episode and the $h$-th step. First, it takes $\tilde{O}(\cO+\cR)$ time to compute the action value function $Q_{k,h}$ in Line~\ref{line:q-function} for a given state-action pair $(s,a)$, since it requires the evaluation of $\hatf_{k,0}$ and the computation of $\cD_\cF$. Then, to take actions based on $\pi^k_h$, $\algomdp$ needs to compute the action value functions $Q_{k,h}$ for $|\cA|$ actions, with each to be computed within $\tilde{O}(\cO+\cR)$ time. Next, $\{\barsigma_{k,h,m}\}_{m\in\seq{M}}$ in Line~\ref{line:home} can be computed within $\tilde{O}(M(\cO+\cR))$ time since they require the evaluation of $\hatf_{k,m}$, $\hatf_{k,m+1}$ and the computation of $\cD_\cF$ for each $\barsigma_{k,h,m}$. Finally, it takes $M\cR$ time to calculate $\{\hatf_{k+1,m}\}_{m\in\seq{M}}$ in Line~\ref{line:regression}. Therefore, the total time cost of $\algomdp$ is $\tilde{O}(KHM\cO+KHM\cR+|\cA|KH\cO+|\cA|KH\cR)$.

\subsection{Regret Bound}
\label{sec:mdp-regret}

\begin{theorem}[Regret]
\label{thm:mdp-regret}
For the episodic MDPs with general function approximation defined in Section~\ref{sec:preliminary}, we set parameters in Algorithm~\ref{algo:mdp} as follows: $M=\lceil \log_2(3KH) \rceil,\lambda=\log\cN_\cF,\sigmamin^2=\sqrt{\dim_\cF \log\cN_\cF}/(KH),\gamma^2=\sqrt{\log\cN_\cF},\epsilon=\log\cN_\cF\cdot\sigmamin^2/(KH)$ and $\{\hatbeta_k\}_{k\ge 1}$ as
\[
    \hatbeta_k = 3\sqrt{\iota_k} + 2\frac{\iota_k}{\gamma^2} + \sqrt{\lambda} + \sqrt{6kH\epsilon/\sigmamin^2}.
\]
Then for any $\delta>0$, with probability at least $1-4(M+1)\delta$, the regret of $\algomdp$ is bounded by 
\begin{equation}
\label{eq:regret}
    \tilde{O}\rbr{\sqrt{\dim_\cF \log\cN_\cF}\cdot\sqrt{\cQ^*} + \dim_\cF \log\cN_\cF},
\end{equation}
where $\iota_k=\tilde{O}(\log\cN_\cF)$, $\dim_\cF=\dim_\cF(\sigmamin,KH)$, $\cN_\cF=\cN_\cF(\epsilon)$, and $\cQ^*=\min\{\cQ_0^*,V^*_1 K\}$ with
\ifarxiv
\begin{align*}
    Q_m = \sum_{k=1}^K\sum_{h=1}^H [\VV (V^*_{h+1})^{2^m}](s^k_h,a^k_h),\quad \cQ_0^* = \max_{m\in\seq{M}}Q_m,\quad V^*_1 = \frac{1}{K}\sum_{k=1}^K V^*_1(s^k_1).
\end{align*}
\else
\begin{align*}
    Q_m &= \sum_{k=1}^K\sum_{h=1}^H [\VV (V^*_{h+1})^{2^m}](s^k_h,a^k_h), \\
    \cQ_0^* &= \max_{m\in\seq{M}}Q_m,\quad V^*_1 = \frac{1}{K}\sum_{k=1}^K V^*_1(s^k_1).
\end{align*}
\fi
\end{theorem}
\begin{proof}
The proof features a novel concentration bound and fine-grained analyses on the higher order expansions of MDPs with general function approximation. See Section~\ref{sec:mdp-proof-sketch} for a proof sketch and Appendix~\ref{proof:mdp} for a detailed proof.
\end{proof}
Our instance-dependent result given by Theorem~\ref{thm:mdp-regret} successfully eliminates the dependence on $H$, with the exception of logarithmic factors. This accomplishment represents a groundbreaking advancement in RL with general function approximation, as it achieves both \emph{horizon-free} and \emph{instance-dependent} for the first time. 

\paragraph{The Quantity $\cQ^*$}
Since $\cQ^*=\min\{\cQ_0^*, V^*_1 K\} \le V_1^* K$, our regret bound in Theorem~\ref{thm:mdp-regret} immediately implies a first-order regret 
$$
\tilde{O}\rbr{\sqrt{\dim_\cF \log\cN_\cF}\cdot\sqrt{ V_1^* K} + \dim_\cF \log\cN_\cF}.
$$
The quantity $\cQ_0^*$ defined in Theorem~\ref{thm:mdp-regret} is the maximum of $Q_m$ over $m\in\seq{M}$, with $Q_m$ being the sum of the variance of $2^m$-th order of the optimal value function along the $k$-th trajectory over $k\in[K]$. Thereby, $\cQ_0^*$ quantifies the stochasticity of the MDP under the optimal policy and it vanishes when the transition kernel is deterministic. In specific, our results also indicate that the regret of learning a deterministic MDP is bounded by $\tilde{O}(\dim_\cF \log\cN_\cF)$. By Proposition~\ref{prop:linear-covering}, this further implies a $\tilde{O}(d^2)$ regret bound for linear mixture MDPs, matching the result in \citet{zhou2022computationally}.

To illustrate our theory more, we present implied regret bounds for linear mixture MDPs.

\begin{corollary}[Regret for Linear Mixture MDPs]
\label{cor:regret-linear}
For $d$-dimensional linear mixture MDPs defined in Assumption~\ref{ass:linear-mixture-mdp}, we set parameters in Algorithm~\ref{algo:mdp} according to Theorem~\ref{thm:mdp-regret}. Then for any $\delta>0$, with probability at least $1-4(M+1)\delta$, the regret of $\algomdp$ is bounded by
\[
    \tilde{O}(d\sqrt{\cQ^*}+d^2) \le \tilde{O}(d\sqrt{V_1^* K}+d^2) \le \tilde{O}(d\sqrt{K}+d^2).
\]
\end{corollary}
\begin{proof}
We have $\dim_\cF, \log\cN_\cF = \tilde{O}(d)$ by Proposition~\ref{prop:linear-eluder}, then the result follows from Theorem~\ref{thm:mdp-regret}.
\end{proof}
Corollary~\ref{cor:regret-linear} demonstrates that we also achieve a first-order regret guarantee for linear mixture MDPs, which covers the state-of-the-art worst-case regret bound $\tilde{O}(d\sqrt{K}+d^2)$ by \citet{zhou2022computationally}, therefore matching the lower bound $\Omega(d\sqrt{K})$ \citep{zhou2021nearly} when $K\ge d^2$ up to logarithmic factors.
The outcome presented in \citet{zhou2022computationally} critically depends on the feature $\phi_V(s,a)\in\RR^d$ defined in Assumption~\ref{ass:linear-mixture-mdp}. This feature endows the value function with a simple linear structure, rendering it amenable to analyze.
Consequently, their proposed algorithms and analytical methodologies encounter limitations when applied to broader function classes characterized by intricate non-linear structures.
We refer readers to Appendix~\ref{sec:comparison-zhou2022} for more explanations.

\subsection{Proof Sketch}
\label{sec:mdp-proof-sketch}

In this section, we provide a proof sketch of the regret bound in Theorem~\ref{thm:mdp-regret}, which primarily relies on the following key lemmas.

\paragraph{Concentration of the Estimated Model}
Our first step is to establish the optimism of the estimated value function as in Line~\ref{line:q-function} of Algorithm~\ref{algo:mdp}, which hinges on a novel concentration inequality in Theorem~\ref{thm:wls}. Notably, this theorem resembles a Bernstein-style bound that substantially expands the scope of weighted linear regression \citep{zhou2021nearly,zhou2022computationally} to non-linear settings, thus accommodating a broader class of functions. We refer readers to Appendix~\ref{sec:comparison-zhou2022} for more technical differences with \citet{zhou2022computationally}.
\begin{theorem}
\label{thm:wls}
Let $\{\cG_t\}_{t\ge1}$ be a filtration, and $\{X_t\}_{t\ge1}, \{Y_t\}_{t\ge1}$ be stochastic processes such that $X_t\in\cX$ is $\cG_{t-1}$-measurable and $Y_t\in[0,L]$ is $\cG_t$-measurable. Let $f_*\in\cF$ with function class $\cF$ consisting of functions $f:\cX\to[0,L]$. Suppose $\EE[Y_t|\cG_{t-1}] = f_*(X_t)$. Let
\begin{equation}
\label{eq:hatf}
    \hatf_{t+1} = \argmin_{f\in\cF} \sum_{s=1}^t w_s^2(f(X_s)-Y_s)^2,
\end{equation}
where the $\cG_{s-1}$-measurable random variable $w_s$ satisfies for all $s\in[t]$, $|w_s|\le W$, $w_s^2\Var[Y_s|\cG_{s-1}]\le\sigma^2$.
Then for any $\delta,\epsilon>0$, with probability at least $1-\delta$, we have for all $t\ge1$,
\ifarxiv
\begin{gather*}
    \sum_{s=1}^t w_s^2(\hatf_{t+1}(X_s)-f_*(X_s))^2 \le \beta_{t+1}^2 \text{ with} \\
    \beta_{t+1} = 3\sqrt{\iota_t}\sigma + 2 \iota_t L \max_{s\in[t]} w_s^2 \cD_\cF(X_s;X_{[s-1]},1/w_{[s-1]}) + \sqrt{\lambda} + \sqrt{6Lt\epsilon/\sigmamin^2},
\end{gather*}
\else
\begin{align*}
    &\sum_{s=1}^t w_s^2(\hatf_{t+1}(X_s)-f_*(X_s))^2 \le \beta_{t+1}^2 \text{ with} \\
    \beta_{t+1} &= 3\sqrt{\iota_t}\sigma + 2 \iota_t L \max_{s\in[t]} w_s^2 \cD_\cF(X_s;X_{[s-1]},1/w_{[s-1]}) \\
    &\qquad+ \sqrt{\lambda} + \sqrt{6Lt\epsilon/\sigmamin^2},
\end{align*}
\fi
where $\cD_\cF$ is in Definition~\ref{def:general-eluder-RL} and $\iota_t=\tilde{O}(\log\cN_\cF)$.
\end{theorem}
\begin{proof}
See Appendix~\ref{proof:wls} for a detailed proof.
\end{proof}
\begin{remark}
Theorem~\ref{thm:wls} provides a variance and uncertainty-aware bound of the deviation of $\hatf_{t+1}$.
In contrast, the bound proposed by \citet{ayoub2020model} for unweighted regression gives $\beta'_{t+1} = \tilde{O}(L\sqrt{\log\cN_\cF})$, which is Hoeffding-type, i.e., it scales with the range of noise. Therefore, UCRL-VTR deviates from achieving an instance-dependent regret bound.
\end{remark}
Recall the estimated model $\hatf_{k+1,m}$ in Line~\ref{line:regression} of Algorithm~\ref{algo:mdp} is the solution of weighted least squares regression. We apply Theorem~\ref{thm:wls} for all $m\in\seq{M}$ with $\{X_t,w_t,Y_t\}_{t\ge 1} = \{z_{i,h,m},1/\barsigma_{i,h,m}$, $y_{i,h,m}\}_{i,h\in[k]\times[H]}$, which gives the confidence radius
\[
    \hatbeta_K = \tilde{O}(\log\cN_\cF).
\]
Then the optimism property of the estimated value function in Line~\ref{line:q-function} can be proved with high probability.

\paragraph{Higher Order Expansion}
The regret bound can be related to the summation of bonuses (see Appendix~\ref{proof:regret}):
\[
    \mathrm{Regret}(K) \lesssim \sum_{k=1}^K\sum_{h=1}^H \min\{1,\hatbeta_k\cD_\cF(z_{k,h,0};D_{k,0})\} + C,
\]
where $\lesssim$ hides constants and $C$ represents lower order terms. Let $R_m,S_m$ denote the summation of bonuses and moments with respect to $m$-th level as follows:
\begin{align*}
    R_m &:= \sum_{k=1}^K\sum_{h=1}^H \min\{1,\hatbeta_k\cD_\cF(z_{k,h,m};D_{k,m})\}, \\
    S_m &:= \sum_{k=1}^K\sum_{h=1}^H [\VV V_{k, h+1}^{2^m}](s_h^k, a_h^k).
\end{align*}
Then we can further bound $R_m$ with Lemma~\ref{lem:higher-expansion}.
\begin{lemma}[Informal]
\label{lem:higher-expansion}
We have for all $m \in \seq{M-1}$,
\begin{equation*}
    R_m \lesssim \hat\beta_K\sqrt{\dim_\cF}\cdot\sqrt{S_m + R_{m+1}} + \hatbeta_K^2\dim_\cF + C,
\end{equation*}
\end{lemma}
Lemma~\ref{lem:higher-expansion} establishes relationships between $R_m$ and $S_m$, the estimated variance to the higher moment. Such a structure recursively expands for $M=O(\log KH)$ times, resulting in a regret bound only log-polynomially dependent on $H$. We further conduct a fine-grained analysis, rendering these quantities instance-dependent:
\begin{lemma}[Informal]
\label{lem:higher-expansion-instance}
We have
\[
    R_0 \lesssim \hatbeta_K\sqrt{\dim_\cF R_0} + \hatbeta_K\sqrt{\dim_\cF}\cdot\sqrt{\cQ^*}+\hatbeta_K^2\dim_\cF + C.
\]
\end{lemma}
The completion of the proof necessitates in-depth analyses of higher-order quantities to eliminate lower-order terms. We refer readers to Appendix~\ref{proof:mdp} for the formal statements of these critical lemmas, as well as an exhaustive and comprehensive proof of Theorem~\ref{thm:mdp-regret}.

\ifarxiv
\section{Experiments}
\else
\section{EXPERIMENTS}
\fi
\label{sec:exp}

In this section, we conduct numerical experiments to demonstrate the effectiveness of our algorithmic designs. We adopt the episodic RiverSwim environment \citep{strehl2008analysis}, which is also considered in \citet{ayoub2020model}.
Figure~\ref{fig:riverswim} shows a $5$-state RiverSwim environment with action space $\cA=\{a_1,a_2\}$. We choose $H\in\{20,100\}$, where $H=20$ in consistency with \citet{ayoub2020model}, and $H=100$ to demonstrate the performance on long planning horizon problems. And we set the number of episodes $K=5000$. We refer readers to Appendix~\ref{sec:exp-appendix} for implementation details and additional experimental results.

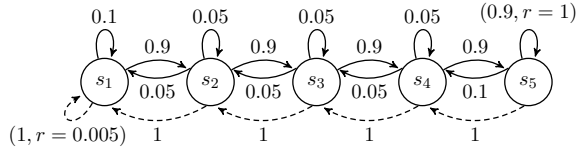
\begin{figure}[ht]
\centering
\resizebox{0.48\textwidth}{!}{
\begin{tikzpicture}[->,>=stealth',shorten >=2pt, 
line width=0.7 pt, node distance=2cm,
scale=1, 
transform shape, align=center, 
state/.style={circle, minimum size=0.3cm, text width=5mm}]
\node[state, draw] (one) {$s_1$};
\node[state, draw] (two) [right of=one] {$s_2$};
\node[state, draw] (three) [right of=two] {$s_3$};
\node[state, draw] (four) [right of=three] {$s_4$};
\node[state, draw] (five) [right of=four] {$s_5$};

\path (one) edge [ loop above ] node {$0.1$} (one) ;
\path (one) edge [ bend left ] node [above]{$0.9$} (two) ;
\draw[->] (two.145) [bend left] to node[below]{$0.05$} (one.35);

\path (two) edge [ loop above ] node {$0.05$} (two) ;
\path (two) edge [ bend left ] node [above]{$0.9$} (three) ;
\draw[->] (three.145) [bend left] to node[below]{$0.05$} (two.35);

\path (three) edge [ loop above ] node {$0.05$} (three) ;
\path (three) edge [ bend left ] node [above]{$0.9$} (four) ;
\draw[->] (four.145) [bend left] to node[below]{$0.05$} (three.35);

\path (four) edge [ loop above ] node {$0.05$} (four) ;
\path (four) edge [ bend left ] node [above]{$0.9$} (five) ;
\draw[->] (five.145) [bend left] to node[below]{$0.1$} (four.35);

\path (five) edge [ loop above ] node {$(0.9, r=1)$} (five);

\path[densely dashed] (one) edge [out=240,in=210,looseness=8] node[below] {$(1, r=0.005)$} (one);
\draw[densely dashed, ->] (two.-100) [bend left] to node[below]{$1$} (one.-80);
\draw[densely dashed, ->] (three.-100) [bend left] to node[below]{$1$} (two.-80);
\draw[densely dashed, ->] (four.-100) [bend left] to node[below]{$1$} (three.-80);
\draw[densely dashed, ->] (five.-100) [bend left] to node[below]{$1$} (four.-80);
\end{tikzpicture}
}
\caption{The transition probabilities and rewards of RiverSwim environment \citep{strehl2008analysis}. The solid arrow shows the transition kernel of action $a_1$. And the dashed arrow represents that of action $a_2$. The reward is denoted with the transition probabilities, if greater than $0$.}
\label{fig:riverswim}
\end{figure}

We mainly consider three approaches to solve the episodic MDPs: (i) our proposed algorithm $\algomdp$; (ii) $\algomdp$ without HOME; (iii) UCRL-VTR \citep{ayoub2020model}. We generate 10 independent paths for each setting. Their average regrets with error bars are shown in Figure~\ref{fig:exp}.

\begin{figure}[th]
\centering
\subfigure[A $5$-state RiverSwim with $H=20$.]{\includegraphics[width=0.48\textwidth]{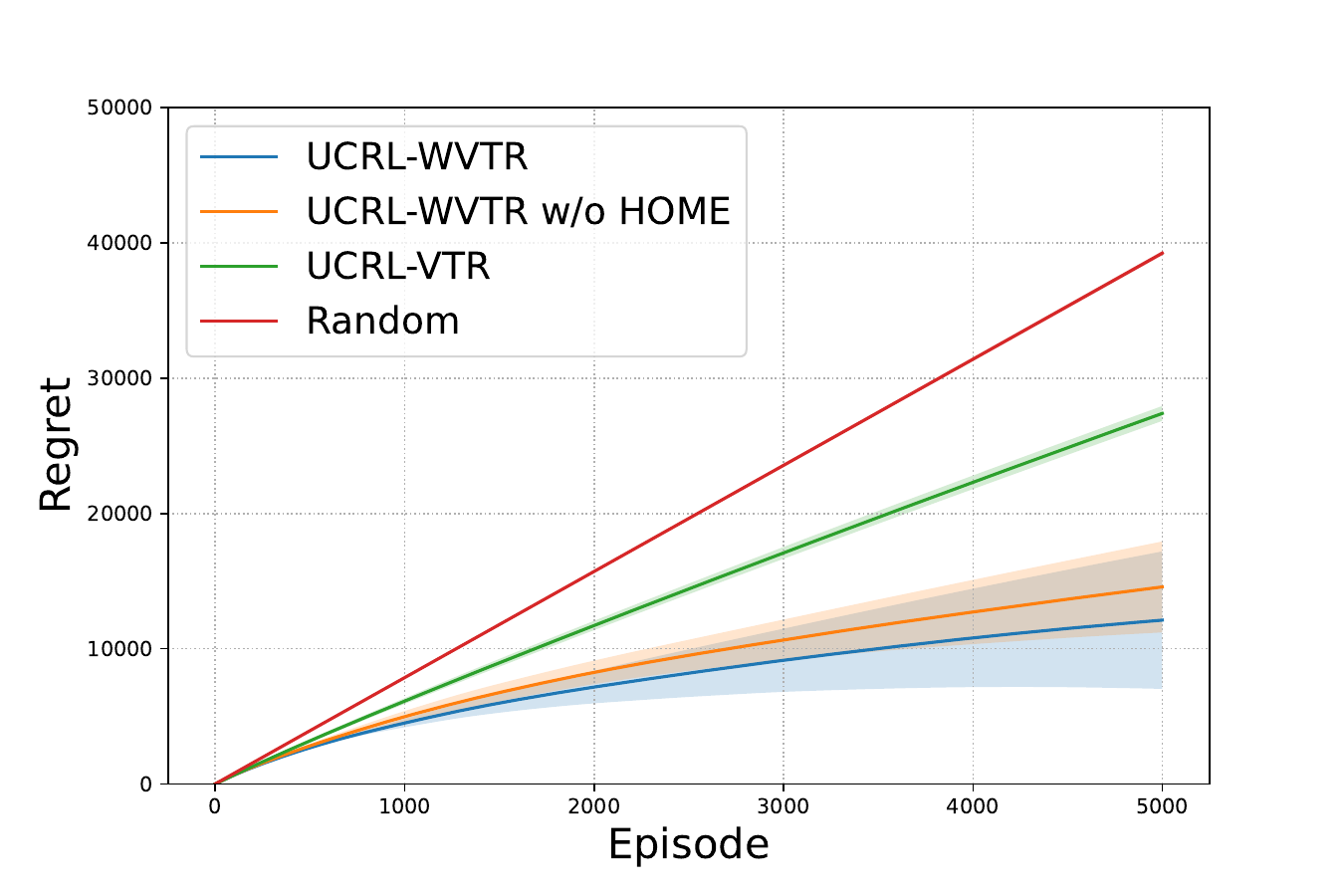}}
\subfigure[A $5$-state RiverSwim with $H=100$.]{\includegraphics[width=0.48\textwidth]{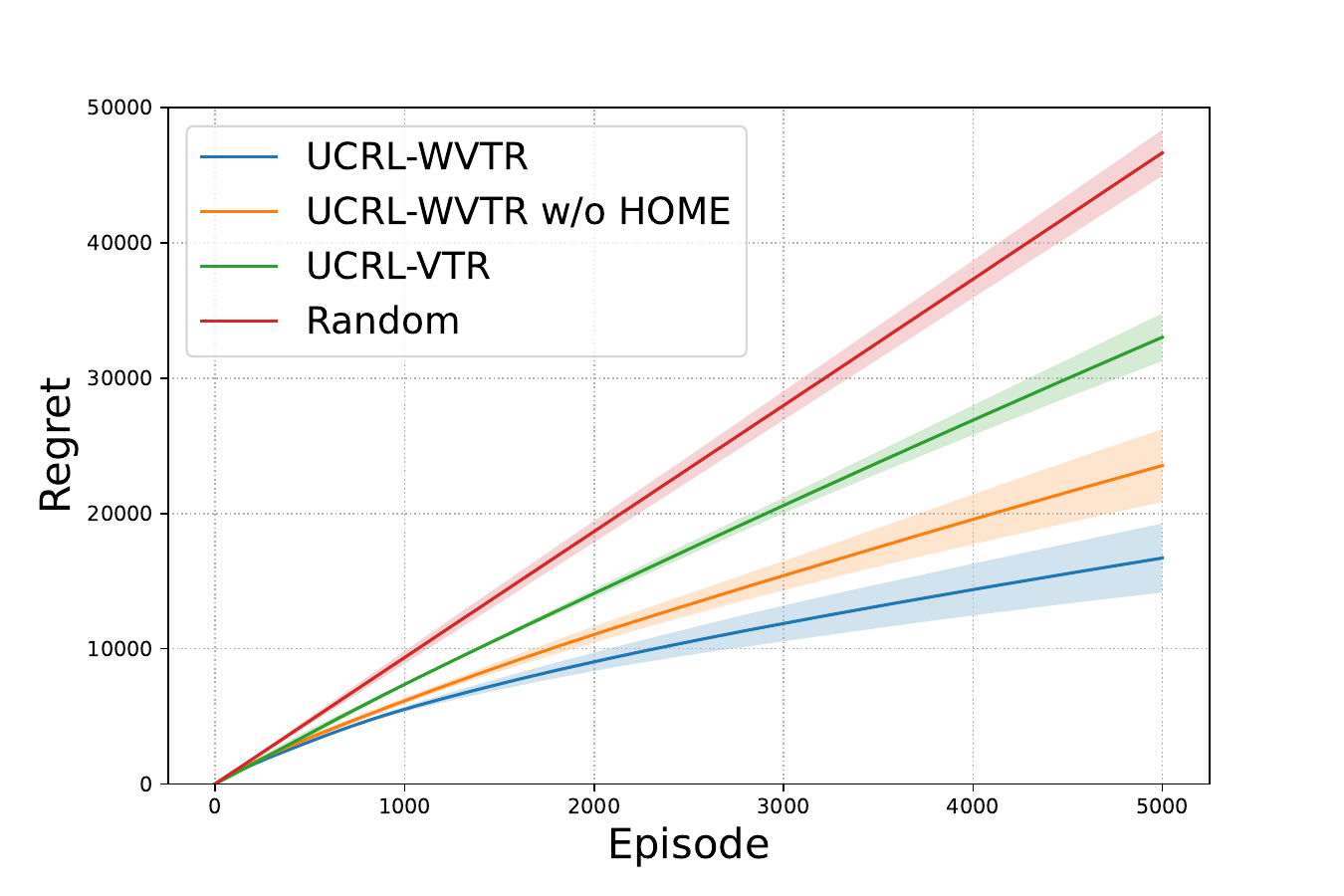}}
\caption{Cumulative regret over episodes. \texttt{Random} represents the baseline where the agent acts uniformly.}
\label{fig:exp}
\end{figure}

As anticipated, $\algomdp$ demonstrates superior performance compared to UCRL-VTR for both settings, underscoring the efficacy of weighted regression. It is noteworthy that the omission of the high-order moment estimator results in a slightly worse performance, thereby illustrating its practical utility. In summary, the experimental findings substantiate and reinforce our theoretical conclusions.

\ifarxiv
\section{Conclusion}
\label{sec:conclu}
In this work, we propose a novel algorithm, termed as $\algomdp$, for model-based RL with general function approximation, which features weighted value-targeted regression and a high-order moment estimator. We theoretically demonstrate it achieves a horizon-free and instance-dependent regret bound. As a special case, it matches the lower bound for linear mixture MDPs up to logarithmic factors, showing its tightness. Furthermore, it is computationally efficient given a regression oracle. We also conduct numerical experiments to validate the theoretical findings. To the best of our knowledge, we are the first to achieve efficient, horizon-free and instance-dependent learning in RL with general function approximation.
\else
\section{CONCLUSION}
\label{sec:conclu}
We propose a novel algorithm for model-based RL with general function approximation. We rigorously prove it achieves a horizon-free and instance-dependent regret bound. Moreover, it is computationally efficient given a regression oracle. Numerical experiments further validate the theoretical findings.
\fi

\section*{Acknowledgments}
Lin F. Yang is supported in part by NSF \#2221871 and an Amazon Faculty Award.

\bibliographystyle{plainnat}
\bibliography{ref}

\clearpage
\appendix
\ifarxiv
\section{Additional Experiments}
\else
\section{ADDITIONAL EXPERIMENTS}
\fi
\label{sec:exp-appendix}

\paragraph{Hyperparameters of Algorithms}

We provide more details about the three algorithms considered for the experiments in Section~\ref{sec:exp}. These approaches share similar structures with Algorithm~\ref{algo:mdp}. (i) $\algomdp$: our proposed algorithm in Algorithm~\ref{algo:mdp}. (ii) $\algomdp$ without $\algovar$: it can be implemented by setting the level $M=1$ in Algorithm~\ref{algo:mdp}. Intuitively, we simply use unweighted regression for solving $f_{k,h,1}$, which is then used for constructing the variance estimate $[\bar{\VV}_{k,0} V_{k,h+1}](s^k_h,a^k_h) = \hatf_{k,1}(z_{k,h,1}) - \hatf_{k,0}^2(z_{k,h,0})$. (iii) UCRL-VTR: it can be implemented by setting the level $M=0$, $\sigmamin=1$ and $\gamma=0$ in Algorithm~\ref{algo:mdp}, since it simply uses unweighted regression, resulting the variance estimation $[\bar{\VV}_{k,0} V_{k,h+1}](s^k_h,a^k_h)=1$. Specifically, we list the hyperparameters in Table~\ref{tab:exp-hyper}.

\begin{table}[ht]
\centering
\caption{Hyperparameters of Algorithms}
\label{tab:exp-hyper}
\begin{tabular}{|c|l|l|l|}
     \hline
     Hyperparameters & $\algomdp$ & \makecell{$\algomdp$ \\ without $\algovar$} & UCRL-VTR \\
     \hline
     $\lambda$ & 0.001 & 0.001 & 0.001 \\
     \hline
     $\sigmamin$ & 0.01 & 0.01 & 1 \\
     \hline
     $\gamma$ & 0.5 & 0.5 & 0 \\
     \hline
     $\beta$ & 1 & 1 & 1 \\
     \hline
     $M$ & 3 & 1 & 0 \\
     \hline
\end{tabular}
\end{table}

\paragraph{Environments}

Recall the RiverSwim environment consists of a series of $n$ states organized in a linear sequence, as shown in Figure~\ref{fig:riverswim-n-state}. The agent's initial position is on the far left, and it faces a decision at each state to either swim to the left or to the right. Notably, there exists a prevailing current that significantly facilitates leftward swimming while making rightward swimming more challenging. When swimming with the current, the agent is guaranteed to make progress to the left. However, swimming against the current predominantly leads to rightward movement but occasionally results in the agent moving left or remaining in the current state. The environment offers rewards only at the extreme ends: a reward of $0.05$ on the far left and a reward of $1$ on the far right. In essence, the agent is compelled to explore the environment thoroughly, ultimately weighing the decision of whether to brave the uncertain prospects of moving against the current for the potentially higher reward, or simply staying in the original position to secure a relatively smaller reward. Consequently, intelligent and strategic exploration becomes a fundamental requirement for acquiring an effective policy in this challenging environment.

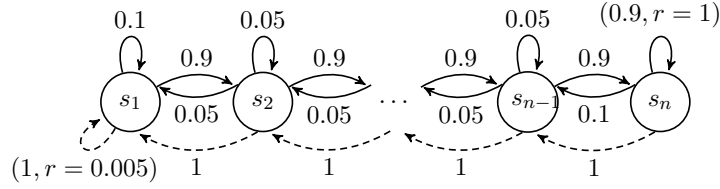
\begin{figure}[ht]
\centering
\resizebox{0.6\textwidth}{!}{
\begin{tikzpicture}[->,>=stealth',shorten >=2pt, 
line width=0.7 pt, node distance=2cm,
scale=1, 
transform shape, align=center, 
state/.style={circle, minimum size=0.3cm, text width=5mm}]
\node[state, draw] (one) {$s_1$};
\node[state, draw] (two) [right of=one] {$s_2$};
\node[state, draw=none] (three) [right of=two] {$\cdots$};
\node[state, draw] (four) [right of=three] {$s_{n-1}$};
\node[state, draw] (five) [right of=four] {$s_{n}$};

\path (one) edge [ loop above ] node {$0.1$} (one) ;
\path (one) edge [ bend left ] node [above]{$0.9$} (two) ;
\draw[->] (two.145) [bend left] to node[below]{$0.05$} (one.35);

\path (two) edge [ loop above ] node {$0.05$} (two) ;
\path (two) edge [ bend left ] node [above]{$0.9$} (three) ;
\draw[->] (three.145) [bend left] to node[below]{$0.05$} (two.35);

\path (three) edge [ bend left ] node [above]{$0.9$} (four) ;
\draw[->] (four.145) [bend left] to node[below]{$0.05$} (three.35);

\path (four) edge [ loop above ] node {$0.05$} (four) ;
\path (four) edge [ bend left ] node [above]{$0.9$} (five) ;
\draw[->] (five.145) [bend left] to node[below]{$0.1$} (four.35);

\path (five) edge [ loop above ] node {$(0.9, r=1)$} (five);

\path[densely dashed] (one) edge [out=240,in=210,looseness=8] node[below] {$(1, r=0.005)$} (one);
\draw[densely dashed, ->] (two.-100) [bend left] to node[below]{$1$} (one.-80);
\draw[densely dashed, ->] (three.-100) [bend left] to node[below]{$1$} (two.-80);
\draw[densely dashed, ->] (four.-100) [bend left] to node[below]{$1$} (three.-80);
\draw[densely dashed, ->] (five.-100) [bend left] to node[below]{$1$} (four.-80);
\end{tikzpicture}
}
\caption{The transition probabilities and rewards of an $n$-state RiverSwim environment \citep{strehl2008analysis}. The solid arrow shows the transition kernel of action $a_1$. And the dashed arrow represents that of action $a_2$. The reward is denoted with the transition probabilities, if greater than $0$.}
\label{fig:riverswim-n-state}
\end{figure}

\paragraph{Additional Results}

To demonstrate the effectiveness of our proposed algorithm more, we show additional experimental results. We set $H=100$ to demonstrate the performance of these algorithms on long planning horizon problems. We will choose a series of gradually increasing the number of states of RiverSwim, i.e., $|S|\in\{6,8,10\}$, which will exponentially increase the difficulty of such an episodic RL problem.

\begin{figure}[t]
\centering
\subfigure[Regret with $|S|=6$]{\includegraphics[width=0.32\textwidth]{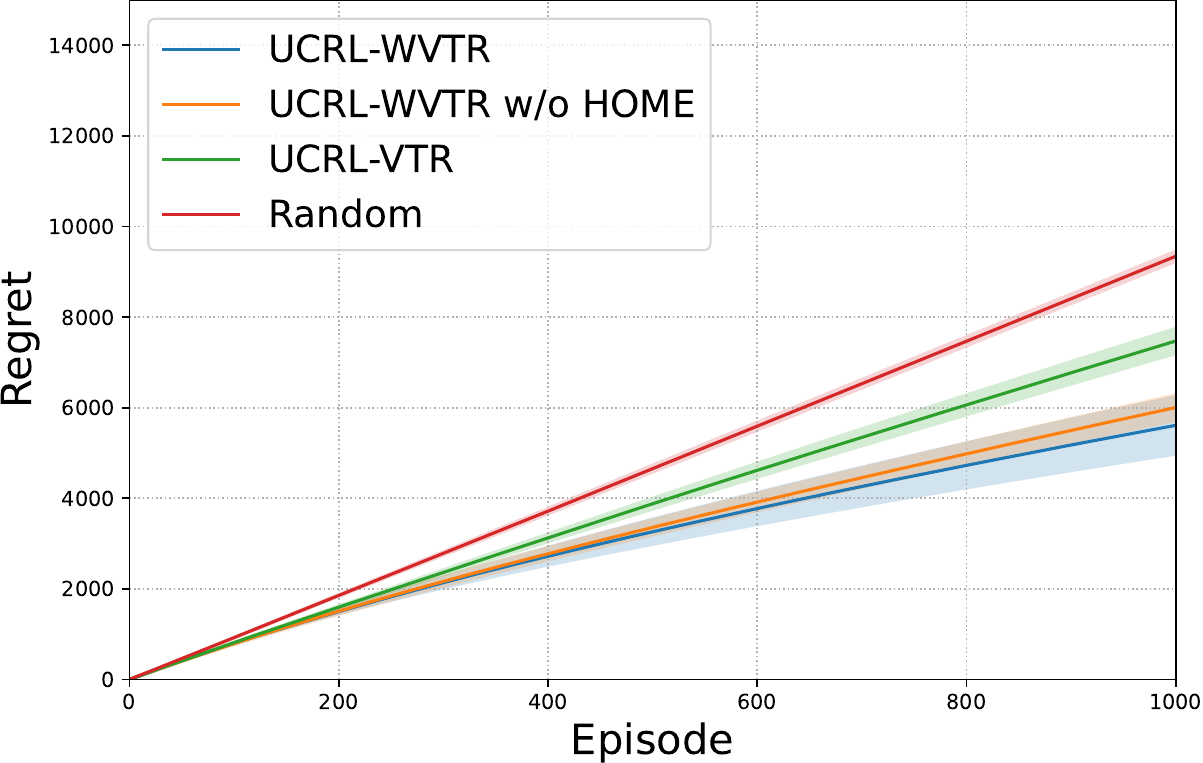}}
\subfigure[Regret with $|S|=8$]{\includegraphics[width=0.32\textwidth]{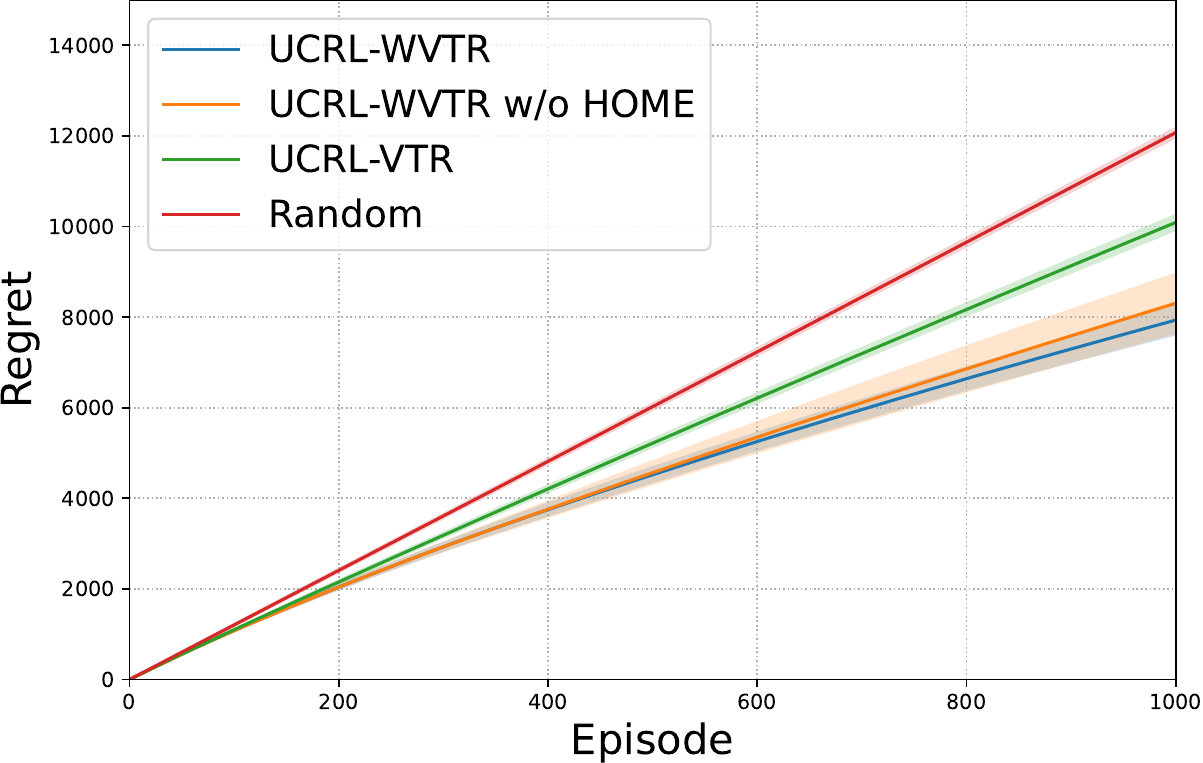}}
\subfigure[Regret with $|S|=10$]{\includegraphics[width=0.32\textwidth]{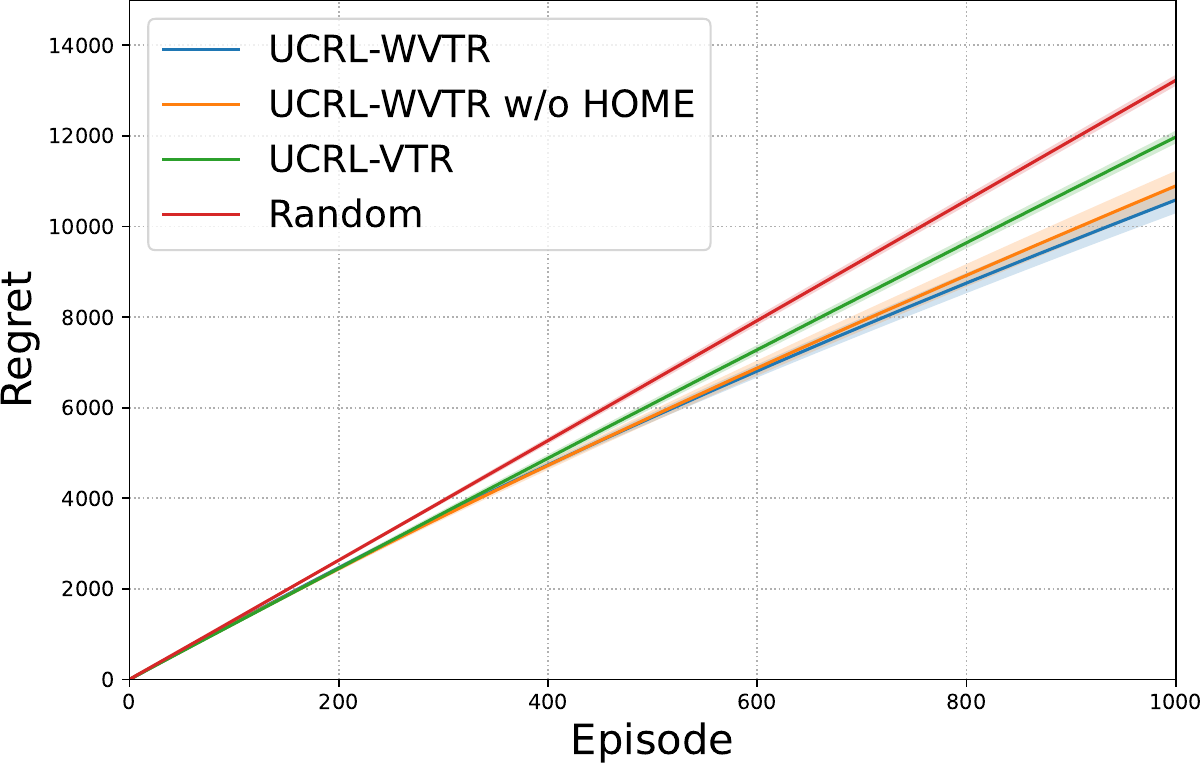}}
\subfigure[Average Rewards with $|S|=6$]{\includegraphics[width=0.32\textwidth]{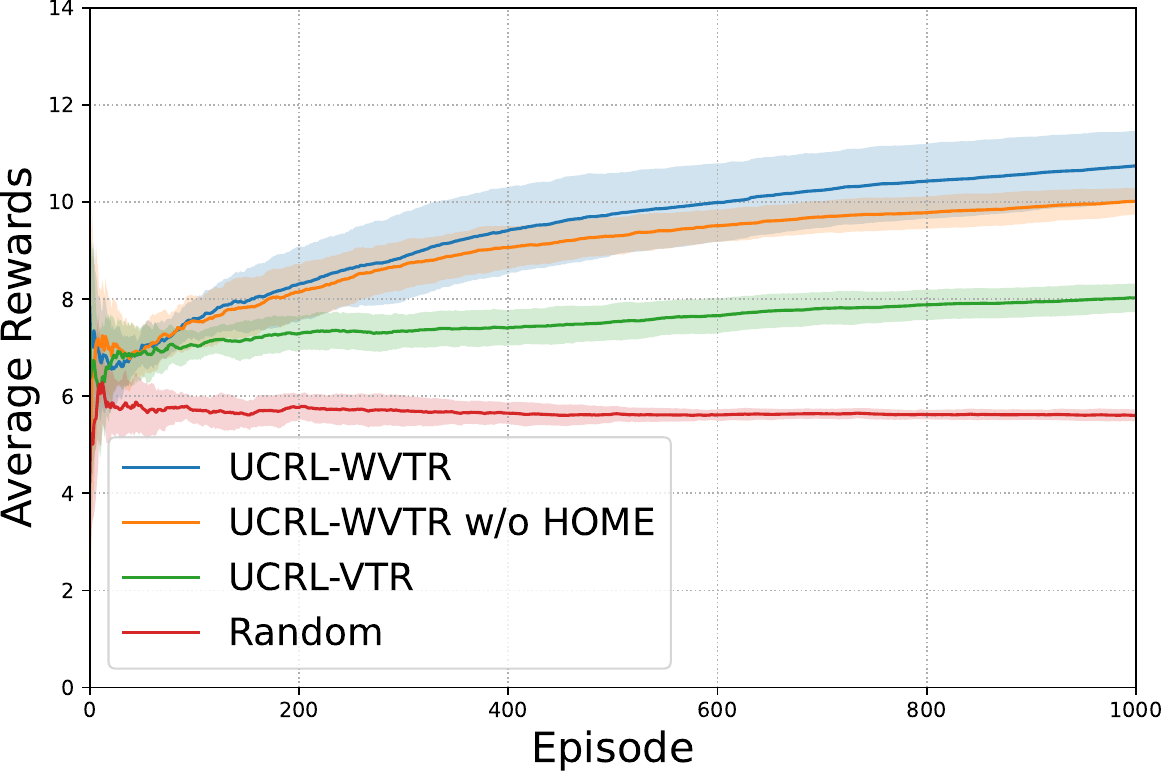}}
\subfigure[Average Rewards with $|S|=8$]{\includegraphics[width=0.32\textwidth]{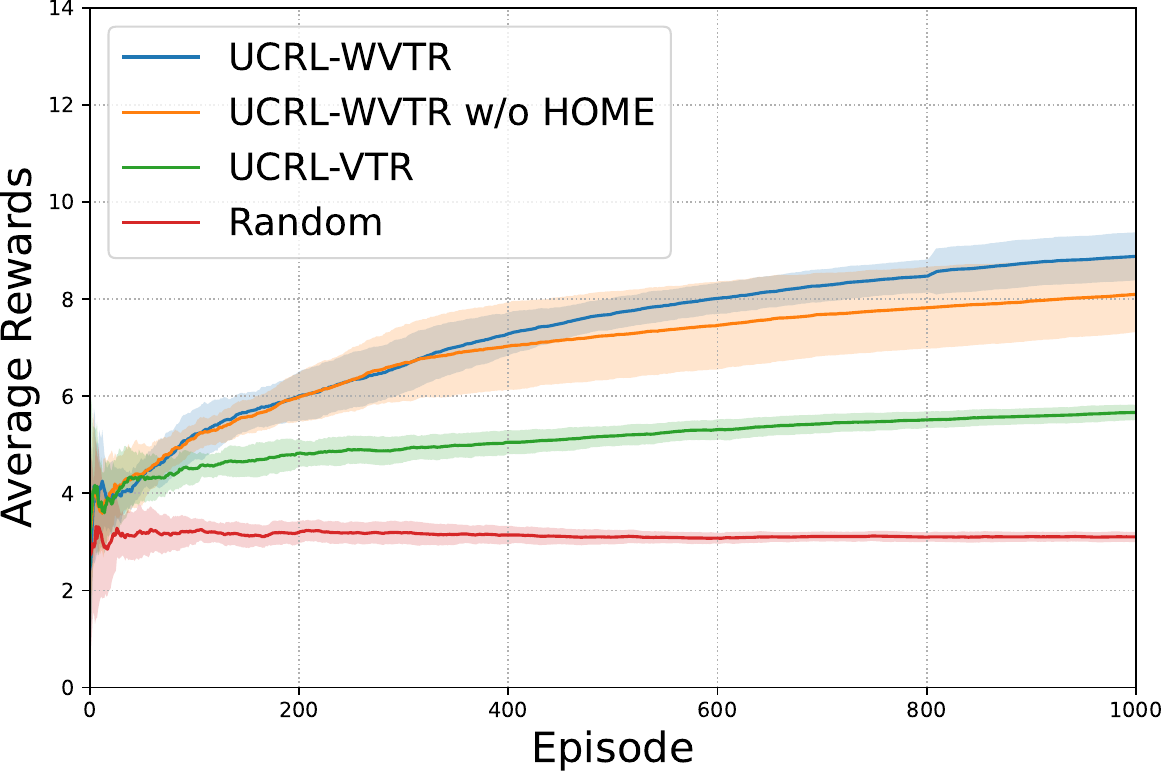}}
\subfigure[Average Rewards with $|S|=10$]{\includegraphics[width=0.32\textwidth]{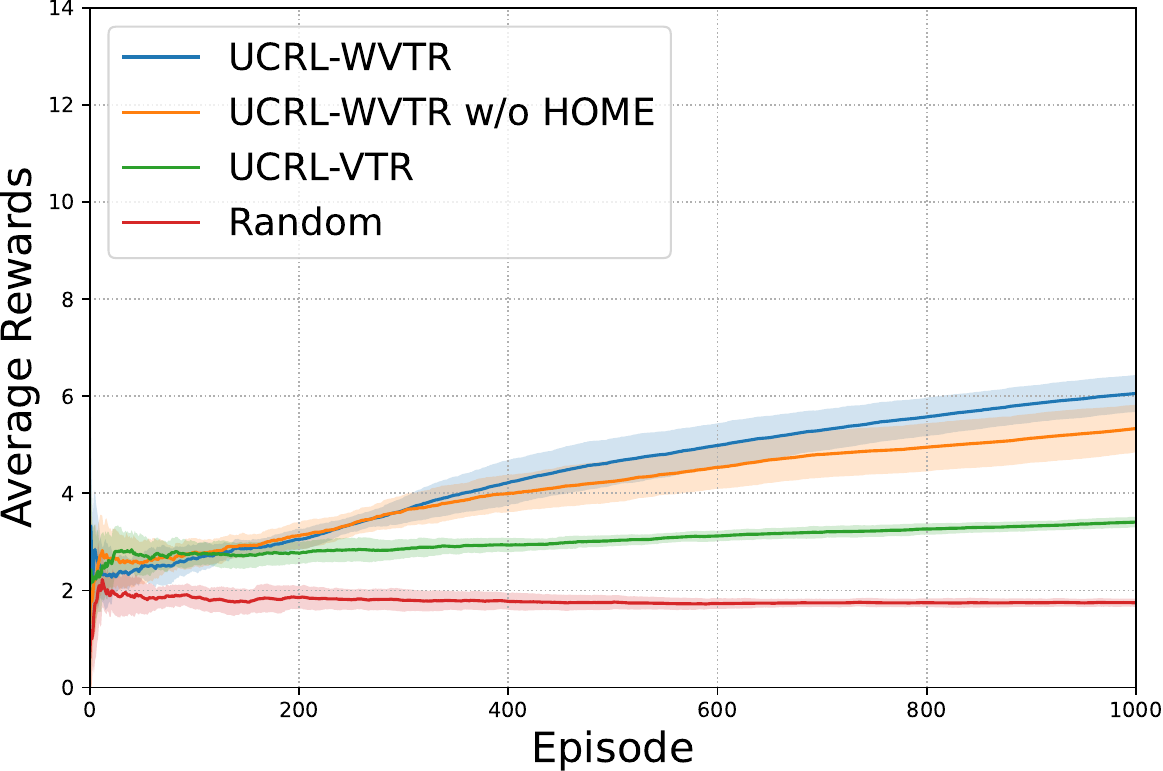}}
\caption{Regret and Average Rewards of $1000$ Episodes}
\label{fig:additional-exp}
\end{figure}

As shown in Figure~\ref{fig:additional-exp}, the columns correspond to RiverSwim with $|S|=6,8,10$, respectively. The first row shows the cumulative regret of each algorithm, and the second row represents average rewards within a single episode. We generate $10$ independent paths for each algorithm, and show the average values plus or minus the standard deviation.

As expected, our proposed algorithm $\algomdp$ outperforms UCRL-VTR, highlighting the effectiveness of weighted regression. It is worth noting that excluding the high-order moment estimator leads to a slightly inferior performance, thus demonstrating its practical value. While it introduces additional uncertainty in estimating higher-order moments, as clearly demonstrated by our results, the extra information from these variance estimations can leverage historical data more effectively, thereby accelerating the learning process. In summary, our experimental results provide strong support for and further reinforce our theoretical conclusions.

\ifarxiv
\section{Computing Uncertainty via the Regression Oracle}
\else
\section{COMPUTING UNCERTAINTY VIA THE REGRESSION ORACLE}
\fi
\label{sec:uncertainty-compute}

In this section, we introduce how to efficiently compute uncertainty $\cD_\cF$ in Definition~\ref{def:general-eluder-RL}. The high-level idea is to decompose the task into solving a series of weighted least squares regression problems, whose solutions are directly available by the regression oracle defined in Assumption~\ref{ass:regression-oracle}. Recall
\begin{equation}
\label{eq:uncertainty-compute-1}
    D^2_\cF(X_t;X_{[t-1]}, \sigma_{[t-1]}) = \sup_{f_1,f_2\in\cF}\frac{\left(f_1(X_t)-f_2(X_t)\right)^2}{\sum_{s=1}^{t-1}\frac{1}{\sigma_s^2}\left(f_1(X_{s})-f_2(X_s)\right)^2+\lambda}.
\end{equation}
We denote $\cG:=\cF-\cF$. Assuming the optimal value of \eqref{eq:uncertainty-compute-1} is attained at $f_1^*,f_2^*$, let $g_*=f_1^*-f_2^*\in\cG$, we claim that $g_*$ is the solution to the following constrained optimization problem:
\begin{equation}
\label{eq:uncertainty-compute-2}
    \max_{g\in\cG} g(X_t) \quad\text{ s.t. }\sum_{s=1}^{t-1} \frac{1}{\sigma_s^2}g^2(X_s) \le \beta^*,
\end{equation}
where $\beta^* = \sum_{s=1}^{t-1} \frac{1}{\sigma_s^2}g_*^2(X_s)$ is unknown in prior. The claim can be easily proved by contradiction. If the solution to \eqref{eq:uncertainty-compute-2} is $\tilde{g}\neq g_*$, then we have $\tilde{g}(X_t)\ge g_*(X_t)$ while $\sum_{s=1}^{t-1} \frac{1}{\sigma_s^2}\tilde{g}^2(X_s) \le \beta^* = \sum_{s=1}^{t-1} \frac{1}{\sigma_s^2}g_*^2(X_s)$, which yields
\[
    \frac{\tilde{g}^2(X_t)}{\sum_{s=1}^{t-1}\frac{1}{\sigma_s^2}\tilde{g}^2(X_s)+\lambda} \ge \frac{g_*^2(X_t)}{\sum_{s=1}^{t-1}\frac{1}{\sigma_s^2}g_*^2(X_s)+\lambda},
\]
a contradiction. Inspired by \citet{foster2018practical}, we then transform the problem \eqref{eq:uncertainty-compute-2} into solving the following weighted regression problem
\[
    \min_{g\in\cG} \sum_{s=1}^t v_s (g(X_s)-Y_s)^2,
\]
where $v_s = 1/\sigma_s^2$ for $s\in[t-1]$, $v_t\in\RR^+$ is some constant unknown to the agent, $Y_s=0$ for $s\in[t-1]$ and $Y_t=1$. The solution is available with access to a regression oracle. We leverage a doubling trick for $\beta^*$ and a binary search technique for the proper value of $v_t$. See Algorithm~\ref{algo:uncertainty} for a detailed description.

\begin{algorithm}[t]
\caption{Computing Uncertainty}
\label{algo:uncertainty}
\begin{algorithmic}[1]
\REQUIRE Function class $\cF$, dataset $\{X_{[t]},\sigma_{[t-1]}\}$, $\lambda$, upper bound $\cB$, precision $\epsilon$ and a regression oracle
\ENSURE The estimated uncertainty $\tilde{\cD}$
\STATE Denote $R(g,v):=\sum_{s=1}^{t-1} v_s (g(X_s)-Y_s)^2 + v (g(X_t)-1)^2$
\STATE Set $\cG \leftarrow \cF-\cF,\bar\beta \leftarrow 2\lambda,\tilde{\cD} \leftarrow 0$
\WHILE{$\bar\beta < 2\cB$}
    \STATE Set $\beta \leftarrow \bar\beta-\lambda$
    \STATE Set $v_L \leftarrow 0$, $v_H \leftarrow 2\beta/\epsilon$
    \STATE Set $g_L \leftarrow 0$, $g_H \leftarrow \argmin_{g\in\cG}R(g,v_H)$
    \STATE Set $z_L \leftarrow 0$, $z_H \leftarrow g_H(X_t)$
    \STATE Set $\Delta \leftarrow \epsilon\beta/4$
    \WHILE{$|z_H-z_L|>\epsilon$ and $|v_H-v_L|>\Delta$}
        \STATE Set $\tilde{v} \leftarrow (v_H+v_L)/2$
        \STATE Set $\tilde{g} \leftarrow \argmin_{g\in\cG}R(g,\tilde{v})$
        \STATE Set $\tilde{z} \leftarrow \tilde{g}(X_t)$
        \IF{$\sum_{s=1}^{t-1}\frac{1}{\sigma_s^2}\tilde{g}^2(X_s)>\beta$}
            \STATE $v_H \leftarrow \tilde{v}$, $z_H \leftarrow \tilde{g}(X_t)$
        \ELSE
            \STATE $v_L \leftarrow \tilde{v}$, $z_L \leftarrow \tilde{g}(X_t)$
        \ENDIF
    \ENDWHILE
    \STATE Set $\tilde{\cD} \leftarrow \max\cbr{\tilde{\cD},\tilde{g}(X_t) / \sqrt{\sum_{s=1}^{t-1}\frac{1}{\sigma_s^2}\tilde{g}^2(X_s)+\lambda}}$
    \STATE Set $\bar\beta = 2\bar\beta$
\ENDWHILE
\end{algorithmic}
\end{algorithm}

\begin{proposition}
\label{prop:uncertainty-compute}
Consider the following optimization problem for solving the uncertainty
\[
    \max_{f_1,f_2\in\cF} \frac{\left(f_1(X_t)-f_2(X_t)\right)^2}{\sum_{s=1}^{t-1}\frac{1}{\sigma_s^2}\left(f_1(X_{s})-f_2(X_s)\right)^2+\lambda}.
\]
By definition, the optimal value is $\cD^2_\cF(X_t;X_{[t-1]}, \sigma_{[t-1]})$, and we assume its solution is $g^*=f_1^*-f_2^*$. For any $\epsilon>0$, we run Algorithm~\ref{algo:uncertainty} to solve the problem above with $\cB$ satisfying $\cB\ge\sum_{s=1}^{t-1} \frac{1}{\sigma_s^2}g_*^2(X_s) + \lambda$. If the function class $\cF$ is convex and closed under point-wise convergence, then Algorithm~\ref{algo:uncertainty} terminates within $O(\log(\cB/\lambda)\log(1/\epsilon))$ calls to the regression oracle and returns $\tilde{\cD}$ such that
\[
    \tilde{\cD}-\epsilon/\sqrt{\lambda} \le D_\cF(X_t;X_{[t-1]}, \sigma_{[t-1]}) \le \sqrt{2}\tilde{\cD}+\epsilon/\sqrt{\lambda}.
\]
\end{proposition}
\begin{proof}
Note $\cG=\cF-\cF$ is convex due to the convexity of $\cF$.
Denote $g_\alpha$ as the $\epsilon$-precision solution of the following constrained optimization problem:
\begin{equation}
\label{eq:eq:uncertainty-compute-3}
    \max_{g\in\cG} g(X_t) \quad\text{ s.t. }\sum_{s=1}^{t-1} \frac{1}{\sigma_s^2}g^2(X_s) \le \lambda 2^\alpha-\lambda.
\end{equation}
Then according to Theorem~1 of \citet{foster2018practical}, $g_\alpha$ is computed within $O(\log(1/\epsilon))$ time, and the estimated uncertainty $\tilde{\cD}$ returned by Algorithm~\ref{algo:uncertainty} is actually the maximum of the following quantities over $\alpha \in \{0,1,\dots,\lceil\log_2(\cB/\lambda)\rceil\}$, as shown in \eqref{eq:eq:uncertainty-compute-4}.
\begin{equation}
\label{eq:eq:uncertainty-compute-4}
    \tilde{\cD} = \max_{\alpha} \frac{g_\alpha(X_t)}{\sqrt{\sum_{s=1}^{t-1} \frac{1}{\sigma_s^2} g_\alpha^2(X_s)+\lambda}}.
\end{equation}
Due to the definition of $\cB$, there exists $\alpha \in \{0,1,\dots,\lceil\log_2(\cB/\lambda)\rceil-1\}$, such that $\lambda2^\alpha \le \beta^*+\lambda \le \lambda2^{\alpha+1}$. On one hand,
\begin{align*}
    \cD_\cF(X_t;X_{[t-1]}, \sigma_{[t-1]}) &= \frac{g_*(X_t)}{\sqrt{\sum_{s=1}^{t-1}\frac{1}{\sigma_s^2}g_*^2(X_s) + \lambda}} = \frac{g_*(X_t)}{\sqrt{\beta^*+\lambda}} \le \frac{g_*(X_t)}{\sqrt{\lambda2^\alpha}} = \sqrt{2} \frac{g_*(X_t)}{\sqrt{\lambda2^{\alpha+1}}}\\
    &\overset{(a)}{\le} \sqrt{2}\frac{g_{\alpha+1}(X_t)}{\sqrt{\sum_{s=1}^{t-1} \frac{1}{\sigma_s^2} g_{\alpha+1}^2(X_s)+\lambda}} + \frac{\epsilon}{\sqrt{\lambda2^\alpha}} \overset{(b)}{\le} \sqrt{2}\tilde{D}+\frac{\epsilon}{\sqrt{\lambda}},
\end{align*}
where $(a)$ holds due to the definition of $g_{\alpha+1}$ in \eqref{eq:eq:uncertainty-compute-3} and $(b)$ holds due to \eqref{eq:eq:uncertainty-compute-4}.
On the other hand, $\cD_\cF(X_t;X_{[t-1]}, \sigma_{[t-1]}) \ge \tilde{D} - \epsilon/\sqrt{\lambda}$ due to the definition of $\cD_\cF$. The proof is completed.
\end{proof}
\begin{corollary}
To invoke Algorithm~\ref{algo:uncertainty} in $\algomdp$, we set $\cB=KH/\sigmamin^2+\lambda$, $\epsilon=\sqrt{\lambda}/(KH)$, then the estimated uncertainty $\tilde{\cD}$ can be derived within $\tilde{O}(1)$ calls to the regression oracle such that $\tilde{\cD} - 1/(KH) \le \cD_\cF \le \sqrt{2}\tilde{\cD} + 1/(KH)$. The error in terms of the estimated uncertainty $\cD_\cF$ only enlarges the regret bound by at most constant terms.
\end{corollary}
\begin{remark}
It is noteworthy that \citet{kong2021online} adopted a similar idea to compute the sensitivity score defined therein via the regression oracle. However, their approach is not directly applicable to ours, since (i) They consider the unweighted regression setting, which is a special case of ours by setting $\sigma_{[t]}=\one$. (ii) The sensitivity score in their work is used for constructing a dataset that approximates historical samples and for applying a low-switching updating scheme, while we use uncertainty to directly update the value function.
\end{remark}

\ifarxiv
\section{Comparisons with Previous Work}
\else
\section{COMPARISONS WITH PREVIOUS WORK}
\fi
\label{sec:comparison-zhou2022}

We make some comparisons with \citet{zhou2022computationally}, with emphasis on why their approaches for linear mixture MDPs defined in Assumption~\ref{ass:linear-mixture-mdp} can not be extended to RL with general function approximation.

First, at the core of their analysis is a Bernstein-style concentration bound (Theorem~4.3 in their work) for the deviation $\|\bmu_t-\bmu^*\|_{\bZ_t}$, where $\bmu_t$ is the solution of least squares regression with (weighted) data $\{\bx_s\}_{s\in[t]}$ and targets $\{y_s\}_{s\in[t]}$, $\bZ_t=\sum_{s=1}^t \bx_s\bx_s^\top+\lambda\bI$. Assuming a filtration $\{\cG_t\}_{t\ge 1}$, here $\bx_t$ is $\cG_{t-1}$-measurable, and $y_t=\dotp{\btheta^*}{\bx_t}+\eta_t$ with noise $\eta_t\in\cG_{t}$. Plugging in the analytical form of the estimate $\bmu_t=\bZ_t^{-1}\sum_{s=1}^t \bx_s y_s$, the following inequality holds
\begin{equation}
\label{eq:zhou2022-self-normalized}
    \|\bmu_t-\bmu^*\|_{\bZ_t} \le \Big\|\sum_{s=1}^t \bx_s \eta_s\Big\|_{\bZ_t^{-1}} + \sqrt{\lambda}\|\bmu^*\|_2.
\end{equation}
Denote $\bd_t=\sum_{s=1}^t \bx_s \eta_s$, $Z_t=\|\bd_t\|_{\bZ_t^{-1}}$. They then decompose $\|\sum_{s=1}^t \bx_s \eta_s\|_{\bZ_t^{-1}}^2 = Z_t^2$ into two bounded martingale difference sequences as in \eqref{eq:zhou2022-decomposition} and apply Freedman's inequality for deriving high probability bounds of these two terms separately.
\begin{equation}
\label{eq:zhou2022-decomposition}
    Z_t^2 \le \sum_{s=1}^t\frac{2\eta_s\bx_s^\top\bZ_{s-1}^{-1}\bd_{s-1}}{1+w_s^2} + \sum_{s=1}^t\frac{\eta_s^2w_s^2}{1+w_s^2},
\end{equation}
where $w_t = \|\bx_t\|_{\bZ_{t-1}^{-1}}$. Such a self-normalized expression in \eqref{eq:zhou2022-self-normalized} holds thanks to the analytical solution $\bmu_t$ of weighted linear regression, and a decomposition in \eqref{eq:zhou2022-decomposition} heavily depends on the linear structure of the estimate $\bmu_t=\bZ_t^{-1}\sum_{s=1}^t \bx_s y_s$. Regrettably, these properties no longer hold in non-linear settings, necessitating a new approach for providing a tight concentration bound on the solution to least squares regression as shown in Theorem~\ref{thm:wls}. We refer the readers to its proof in Appendix~\ref{proof:wls} for details.

Second, once deriving a tight estimate of model with $\|\hat{\btheta}_k-\btheta^*\|_{\hat{\bSigma}_{k}} \le \hatbeta_k$, given any state-action pair $(s,a)$ and the next-state value function $V$, one can naturally construct an upper bound of the next-state value function evaluated on $(s,a)$ by explicitly adding a bonus of the following form
\begin{equation}
\label{eq:zhou2022-bonus}
    |\hatf_{k}(s,a,V)-f_*(s,a,V)| = \|\dotp{\bphi_V(s,a)}{\hat{\btheta}_k-\btheta_*}\| \le \hatbeta_k \|\bphi_V(s,a)\|_{\hat{\bSigma}_{k}^{-1}},
\end{equation}
where $\hat{\bSigma}_k = \sum_{i=1}^{k-1}\sum_{h=1}^H \frac{1}{\barsigma_{i,h}^2} \bphi_{i,h}\bphi_{i,h}^\top + \lambda\bI$ with $\bphi_{k,h} = \bphi_{V_{k,h+1}}(s^k_h,a^k_h)$. Then, adopting the policy that selecting action with the maximum value function, and using the standard technique, one can approximately bound the $K$-episode regret as
\begin{equation}
\label{eq:zhou2022-regret}
    \mathrm{Regret}(K)
    \lesssim \hatbeta_{K} \sum_{k=1}^K\sum_{h=1}^H \|\bphi_{k,h}\|_{\hat{\bSigma}_{k}^{-1}}
    \lesssim \hatbeta_{K} \sqrt{\sum_{k=1}^K\sum_{h=1}^H \barsigma_{k,h}^2} \sqrt{\sum_{k=1}^K\sum_{h=1}^H \nbr{\frac{\bphi_{k,h}}{\barsigma_{k,h}}}_{\hat{\bSigma}_{k}^{-1}}^2}.
\end{equation}
Here $\sum_{k=1}^K\sum_{h=1}^H \nbr{\frac{\bphi_{k,h}}{\barsigma_{k,h}}}_{\hat{\bSigma}_{k}^{-1}}^2$ can be easily bounded by $\tilde{O}(d)$ by elliptical potential lemma \citep{abbasi2011improved}. However, in the realm of RL with general function approximation, the bonus term does not enjoy such a simple form as in \eqref{eq:zhou2022-bonus}, and the elliptical potential lemma is not available. To this end, we define the bonus as an optimization problem over the general function class, where an analytical solution is generally not attainable. To overcome this obstacle, we further introduce a computationally efficient algorithm for computing uncertainty for the bonus term via a regression oracle, thus achieving computational efficiency despite the non-linear structure. We then carefully bound the regret with respect to the generalized Eluder dimension. Please see our definition of the uncertainty and the generalized Eluder dimension in Definition~\ref{def:general-eluder-RL}, and computational complexity in Section~\ref{sec:computational-complexity}.

Furthermore, they leverage a high-moment estimator for a tight upper bound of $\sum_{k=1}^K\sum_{h=1}^H \barsigma_{k,h}^2$ in \eqref{eq:zhou2022-regret}. More specifically, they estimate the variances of the $m$-th moment of the next-state value functions up to $M=O(\log(KH))$ levels, and recursively establish the relationships between quantities of MDPs of these levels. Consequently, they obtain an upper bound of the variance term with only log-polynomial dependence on $H$, thus ultimately achieving a horizon-free regret bound. However, these lemmas with respect to the higher-order expansions of MDPs may not hold in the context of general function approximation. Despite the inherent complexities posed by the non-linear function class, we have extended these methodologies to achieve a horizon-free regret bound with log-polynomial dependence on the planning horizon $H$. Significantly, our refined analysis also establishes a connection with certain problem-specific quantities, consequently rendering the regret bound of $\algomdp$ instance-dependent. Comprehensive elaboration on this achievement is provided in Section~\ref{sec:mdp-regret}.

\ifarxiv
\section{Proofs for Linear mixture MDPs}
\else
\section{PROOFS FOR LINEAR MIXTURE MDPS}
\fi
\label{proof:linear}

\subsection{Proof of Proposition~\ref{prop:linear-eluder}}
\label{proof:linear-eluder}
\begin{proof}[Proof of Proposition~\ref{prop:linear-eluder}]
We first bound the covering number, then come to the generalized Eluder dimension.

\paragraph{Covering Number}
For any $f\in\cF$, we have for any $(s,a,V)$,
\begin{align}
    f(s,a,V) &= [\PP V](s,a) = \sum_{s'\in\cS}\dotp{\btheta^*}{\bphi(s'|s,a)}V(s') \notag \\
    &= \dotp{\btheta^*}{\sum_{s'\in\cS}\bphi(s'|s,a)V(s')} = \dotp{\btheta^*}{\bphi_V(s,a)}, \label{eq:linear-representation}
\end{align}
where $\|\btheta^*\|_2\le B$. This implies $f$ is a linear function of $\bphi_V(s,a)$ and the $\epsilon$-log-covering number of $\cF$ follows from standard covering-number arguments.

\paragraph{Generalized Eluder Dimension}
First, we show uncertainty $\cD_\cF$ has an analytic expression for linear mixture MDPs in the following lemma.

\begin{lemma}
For $d$-dimensional, $B$-bounded linear mixture MDPs defined in Assumption~\ref{ass:linear-mixture-mdp}, given any $\bX=\{(s_t,a_t,V_t)\}_{t\in[T]}$, $\bsigma=\{\sigma_t\}_{t\in[T]}$, we have
\[
    \cD_\cF(X_t;X_{[t-1]},\sigma_{[t-1]})=\|\bphi_t\|_{\bSigma_{t-1}},
\]
where $\bphi_t=\bphi_{V_t}(s_t,a_t)$ and $\bSigma_t=\sum_{s=1}^t \frac{1}{\sigma_s^2}\bphi_s\bphi_s^\top + \lambda/B\cdot\bI$ for all $t\in[T]$.
\end{lemma}
\begin{proof}
For any $f_1,f_2\in\cF$, let $\btheta\in\RR^d$ satisfies for any $X=(s,a,V)$,
\[
    f_1(X)-f_2(X)=\dotp{\btheta}{\bphi_V(s,a)},
\]
where the linear representation of $f_1$ and $f_2$ follows from \eqref{eq:linear-representation}. Then we have
\[
    \frac{(f_1(X_t)-f_2(X_t))^2}{\sum_{s=1}^{t-1}\frac{1}{\sigma_s^2}(f_1(X_s)-f_2(X_s))^2 + \lambda}
    \overset{(a)}{\le} \frac{\dotp{\btheta}{\bphi_t}^2}{\sum_{s=1}^{t-1}\frac{1}{\sigma_s^2}\dotp{\btheta}{\bphi_s}^2+\frac{\lambda}{B}\|\btheta\|_2^2}
    \overset{(b)}{=}{} \frac{\dotp{\btheta}{\bphi_t}^2}{\|\btheta\|_{\bSigma_{t-1}}^2}
    \overset{(c)}{\le} \|\bphi_t\|_{\bSigma_{t-1}^{-1}}^2,
\]
where $(a)$ holds due to $\|\btheta\|_2\le B$, $(b)$ holds due to the definition of $\bSigma_{t-1}$ and $(c)$ holds due to Cauchy-Schwarz inequality. Thus $\cD_\cF(X_t;X_{[t-1]},\sigma_{[t-1]}) = \|\bphi_t\|_{\bSigma_{t-1}^{-1}}$ by Definition~\ref{def:general-eluder-RL}.
\end{proof}

Next, we come to the generalized Eluder dimension. We further assume for all $t\in[T]$, $\sigma_t \ge \sigmamin$, then it follows that
\[
    \sum_{t=1}^T \min\cbr{1,\frac{1}{\sigma_t^2}\cD_\cF(X_t;X_{[t-1]},\sigma_{[t-1]})}
    = \sum_{t=1}^T \min\cbr{1,\frac{1}{\sigma_t^2}\|\bphi_t\|_{\bSigma_{t-1}^{-1}}^2}
    \le 2d\log\rbr{1+\frac{TB}{d\lambda\sigmamin^2}},
\]
where the inequality holds due to Lemma~\ref{lem:elliptical}. Taking maximum of the LHS over $|\bX|=T, \bsigma\ge\sigmamin$ and the proof is completed by Definition~\ref{def:general-eluder-RL}.
\end{proof}

\ifarxiv
\section{Proof of Theorem~\ref{thm:wls}}
\else
\section{PROOF OF THEOREM~\ref{thm:wls}}
\fi
\label{proof:wls}
\begin{proof}[Proof of Theorem~\ref{thm:wls}]
Recall the definition of $\hat{f}_{t+1}$ in \eqref{eq:hatf}, which implies
\[
    \sum_{s=1}^t w_s^2(\hat{f}_{t+1}(X_s)-f_*(X_s))^2 \le 2 \sum_{s=1}^t w_s^2(\hat{f}_{t+1}(X_s)-f_*(X_s))(Y_s - f_*(X_s)).
\]
For any fixed $f\in\cF$, denote $E_s(f) = w_s^2(f(X_s)-f_*(X_s))(Y_s - f_*(X_s))$, which is a martingale difference sequence adapted to the filtration $\{\cG_s\}_{s\in[t]}$. Inspired by \citet{agarwal2023vo}, we are to give a high-probability bound of $\sum_{s=1}^t E_s(f)$ in terms of $\sum_{s=1}^t w_s^2(f(X_s)-f_*(X_s))^2$. Note $|w_s|\le W$ and $f(X_s),f_*(X_s),Y_s$ are bounded in $[0,L]$, thereby the expectation and summation of variances are upper bounded by
\begin{align*}
    |E_s(f)| \le W^2 L^2,\qquad
    \sum_{s=1}^t \EE[E_s^2(f)|\cG_{s-1}] \le \sigma^2 W^2 L^2 t.
\end{align*}
Simultaneously, the following problem-specific bounds hold:
\begin{align*}
    \max_{s\in[t]} |E_s(f)| &\overset{(a)}{\le} L \max_{s\in[t]} w_s^2 \cD_\cF(X_s;X_{[s-1]},1/w_{[s-1]}) \sqrt{\sum_{i=1}^{s-1} w_i^2(f(X_i)-f_*(X_i))^2 + \lambda} \\
    &\le L \max_{s\in[t]} w_s^2 \cD_\cF(X_s;X_{[s-1]},1/w_{[s-1]})\cdot\sqrt{\sum_{s=1}^{t-1} w_s^2(f(X_s)-f_*(X_s))^2 + \lambda},\\
    \sum_{s=1}^t \EE[E_s^2(f)|\cG_{s-1}] &\le \sigma^2\sum_{s=1}^t w_s^2(f(X_s)-f_*(X_s))^2,
\end{align*}
where $(a)$ holds due to the definition of $\cD_\cF$ in Definition~\ref{def:general-eluder-RL}. We denote $\cD_s = \cD_\cF(X_s;X_{[s-1]},1/w_{[s-1]})$ for short. Let $\epsilon>0$ be a constant and $\cV$ be a $\epsilon$-covering net of $\cF$. Applying Lemma~\ref{lem:freedman-variant} with $m=v=\sigma^2$, $\iota_t = 16\log\frac{2\cN_\cF(\epsilon)t^2 (\log(\sigma^2 W^2 L^2 t)+2) (\log(W^2 L^2)+2)}{\delta}$, together with a union bound over $f\in\cV$, with probability at least $1-\delta/(2t^2)$, we have for any fixed $t\ge1$ and all $f\in\cV$,
\begin{align*}
    2\sum_{s=1}^t E_s(f) &\le \sqrt{\iota_t}\cdot\sqrt{\sigma^2\sum_{s=1}^t w_s^2(f(X_s)-f_*(X_s))^2 + \sigma^4} \\
    &\qquad+ \iota_t \rbr{L \max_{s\in[t]} w_s^2 \cD_s\cdot\sqrt{\sum_{s=1}^{t-1} w_s^2(f(X_s)-f_*(X_s))^2 + \lambda} + \sigma^2} \\
    &\overset{(a)}{\le} \rbr{\sqrt{\iota_t}\sigma + \iota_t L \max_{s\in[t]} w_s^2 \cD_s} \sqrt{\sum_{s=1}^t w_s^2(f(X_s)-f_*(X_s))^2} + \sqrt{\lambda}\iota_t L \max_{s\in[t]} w_s^2 \cD_s + 2\iota_t\sigma^2 \\
    &\overset{(b)}{\le} \frac{1}{2} \sum_{s=1}^t w_s^2(f(X_s)-f_*(X_s))^2 + \frac{1}{2}\rbr{\sqrt{\iota_t}\sigma + \iota_t L \max_{s\in[t]} w_s^2 \cD_s}^2 \\
    &\qquad+ \frac{1}{2}\rbr{\iota_t L \max_{s\in[t]} w_s^2 \cD_s}^2 + \frac{1}{2}\lambda + 2\iota_t\sigma^2 \\
    &\overset{(c)}{\le} \frac{1}{2} \sum_{s=1}^t w_s^2(f(X_s)-f_*(X_s))^2 + \frac{1}{2}\rbr{\sqrt{\iota_t}\sigma + 2\iota_t L \max_{s\in[t]} w_s^2 \cD_s}^2 + \frac{1}{2}\lambda + 2\iota_t\sigma^2,
\end{align*}
where $(a)$ holds due to $\sqrt{a+b}\le\sqrt{a}+\sqrt{b}$ for any $a,b\ge0$ and $\iota_t\ge1$, $(b)$ holds due to $\sqrt{ab}\le a/2+b/2$ for any $a,b\ge0$, and $(c)$ holds due to $a^2+b^2\le\rbr{a+b}^2$ for any $a,b\ge0$.
Let $g\in\cV$ such that $\|g - \hat{f}_{t+1}\|_\infty \le \epsilon$, then
\begin{align*}
    &\sum_{s=1}^t w_s^2(\hat{f}_{t+1}(X_s)-f_*(X_s))^2\\
    \le{}& 2\sum_{s=1}^t w_s^2(\hat{f}_{t+1}(X_s)-f_*(X_s))(Y_s - f_*(X_s))\\
    \le{}& 2\sum_{s=1}^t w_s^2(g(X_s)-f_*(X_s))(Y_s - f_*(X_s)) + \frac{2Lt}{\sigmamin^2}\epsilon\\
    \le{}& \frac{1}{2} \sum_{s=1}^t w_s^2(g(X_s)-f_*(X_s))^2 + \frac{1}{2}\rbr{\sqrt{\iota_t}\sigma + 2\iota_t L \max_{s\in[t]} w_s^2 \cD_s}^2 + \frac{1}{2}\lambda + 2\iota_t\sigma^2 + \frac{2Lt}{\sigmamin^2}\epsilon\\
    \le{}& \frac{1}{2} \sum_{s=1}^t w_s^2(\hatf_{t+1}(X_s)-f_*(X_s))^2 + \frac{1}{2}\rbr{\sqrt{\iota_t}\sigma + 2\iota_t L \max_{s\in[t]} w_s^2 \cD_s}^2 + \frac{1}{2}\lambda + 2\iota_t\sigma^2 + \frac{3Lt}{\sigmamin^2}\epsilon.
\end{align*}
That is for any fixed $t>0$, we have
\begin{align*}
    \sum_{s=1}^t w_s^2(\hat{f}_{t+1}(X_s)-f_*(X_s))^2 &\le \rbr{\sqrt{\iota_t}\sigma + 2\iota_t L \max_{s\in[t]} w_s^2 \cD_s}^2 + \lambda + 4\iota_t\sigma^2 + \frac{6Lt}{\sigmamin^2}\epsilon \\
    &\le \rbr{3\sqrt{\iota_t}\sigma + 2\iota_t L \max_{s\in[t]} w_s^2 \cD_s + \sqrt{\lambda} + \sqrt{\frac{6Lt\epsilon}{\sigmamin^2}}}^2,
\end{align*}
where the second inequality holds due to $2\sqrt{ab}\le a+b$ and $a+b\le(\sqrt{a}+\sqrt{b})^2$ for any $a,b\ge0$.
Finally, the result holds through a union bound over all $t\ge1$ and the fact that $\sum_{t=1}^\infty \frac{1}{2t^2}\le1$.
\end{proof}

\ifarxiv
\section{Proof of Theorem~\ref{thm:mdp-regret}}
\else
\section{PROOF OF THEOREM~\ref{thm:mdp-regret}}
\fi
\label{proof:mdp}

We define filtration $\{\cG_{k,h}\}_{k\in[K],h\in[H]}$ as follows. Let $\cG_{k,h}$ be the $\sigma$-field generated by the state-action pairs up to $h$-th step and $k$-th episode. That is $\cG_{k,h} = \sigma(\{s^i_h,a^i_h\}_{(i,h)\in{[k-1]\times[H]}} \cup \{s^k_j,a^k_j\}_{j\in[h]})$. For constants $\sigmamin,\epsilon>0$, let $\cN_\cF=\cN_\cF(\epsilon)$, and $\dim_\cF:=\dim_\cF(\sigmamin,KH)$ be the generalized Eluder dimension in Definition~\ref{def:general-eluder-RL}.

\subsection{High Probability Events and Optimism}
\label{proof:event}

Let $\hat\cB_{k,m}$ denote the confidence region as follows:
\begin{equation}
\label{eq:hatcB}
    \hat\cB_{k,m} := \cbr{f:\sum_{i=1}^{k-1} \sum_{h=1}^H \frac{1}{\barsigma_{i,h,m}^2} (\hatf_{k,h,m}(z_{i,h,m})-f(z_{i,h,m}))^2 \le \hatbeta_k^2},
\end{equation}
where
\begin{equation}
\label{eq:hatbeta}
    \hatbeta_k = 3\sqrt{\iota_k} + 2\frac{\iota_k}{\gamma^2} + \sqrt{\lambda} + \sqrt{6kH\epsilon/\sigmamin^2}
\end{equation}
with $\iota_k = 16\log\frac{2\cN_\cF(\epsilon)k^2H^2 (\log(kH/\sigmamin^2)+2) (\log(1/\sigmamin^2)+2)}{\delta}$.
We define event $\cE$ as
\[
    \cE := \{f_* \in \hat\cB_{k,m}\text{ for all }k\in[K],m\in\seq{M}\}.
\]
The following lemmas hold.

\begin{lemma}
\label{lem:event-bonus}
For any $k\in[K],m\in[M]$, if $f_*\in\hat\cB_{k,m}$, we have for all $h\in[H]$,
\[
    |\hatf_{k,m}(z_{k,h,m}) - f_*(z_{k,h,m})| \le \hatbeta_{k}\cD_\cF(z_{k,h,m};D_{k,m}).
\]
Furthermore,
\begin{align*}
    V_{k,h}(s^k_h) - r_h(s^k_h,a^k_h) - [\PP V_{k,h+1}](s^k_h,a^k_h) &\le 2\min\{1,\hatbeta_k\cD_\cF(z_{k,h,0};D_{k,0})\}, \\
    |[\bar{\VV}_{k,h}V_{k,h+1}^{2^m}](s^k_h,a^k_h)-[\VV_{k,h}V_{k,h+1}^{2^m}](s^k_h,a^k_h)| &\le E_{k,h,m}.
\end{align*}
\end{lemma}
\begin{proof}
See Appendix~\ref{proof:event-bonus} for a detailed proof.
\end{proof}

\begin{lemma}
\label{lem:event-concentration}
Event $\cE$ holds with probability at least $1-(M+1)\delta$.
\end{lemma}
\begin{proof}
See Appendix~\ref{proof:event-concentration} for a detailed proof.
\end{proof}

\begin{lemma}
\label{lem:optimism}
On event $\cE$, we have for all $k,h\in[K]\times[H]$, $Q_{k,h}(\cdot,\cdot)\ge Q^*_h(\cdot,\cdot)$, $V_{k,h}(\cdot)\ge V^*_h(\cdot)$.
\end{lemma}
\begin{proof}
See Appendix~\ref{proof:optimism} for a detailed proof.
\end{proof}

\subsection{Higher Order Expansion of MDPs}
\label{proof:higher-order}

Inspired by \citet{zhang2021improved,zhou2022computationally,zhao2023variance}, we define the following quantities of MDPs. For all $k\in[K],h\in[H+1]$, We use $I^k_h$ to denote the following events:
\begin{equation}
\label{eq:def:ind}
    I^k_h:=\ind\cbr{\forall m\in\seq{M}, \sum_{j=1}^{h-1} \frac{1}{\barsigma_{k,j,m}^2}\cD_\cF^2(z_{k,j,m};D_{k,j-1,m})\le1}.
\end{equation}
Note $I^k_h$ is $\cG_{k,h-1}$-measurable and monotonically decreasing.
We define $h_k$ for all $k\in[K]$ as the least $h$ such that $I^k_h$ vanishes.
\begin{equation}
\label{eq:def:h_k}
    h_k=\min\{h\in[H+1]|I^k_h=0\}.
\end{equation}
We use the quantity $G$ to denote the number of episodes when the uncertainty quantity grows sharply:
\begin{align} 
    G = \sum_{k=1}^K(1-I^k_{H+1}). \label{eq:def:g}
\end{align}
We use $\checkV_{k, h}(s)$ to denote the estimation error between the estimated value function and the optimal value function, and use $\tildeV_{k, h}(s)$ to denote the sub-optimality gap of policy $\pi^k$ at stage $h$:
\begin{align}
    & \checkV_{k, h}(s) = V_{k, h}(s) - V_h^*(s), \quad \forall s \in \cS, (k, h) \in [K]\times[H], \label{eq:def:checkv}\\
    & \tildeV_{k, h}(s) = V_h^*(s) - V_h^{\pi^k}(s), \quad \forall s \in \cS, (k, h) \in [K]\times[H]. \label{eq:def:tildev}
\end{align}
In addition, we use $\check{S}_m,\tilde{S}_m,S_m,Q_m$ to represent the total variance of $2^m$-th order value functions ($\checkV_{k, h+1}^{2^m}$, $\tildeV_{k, h+1}^{2^m}$, $V_{k, h+1}^{2^m}$, $(V^*_{h+1})^{2^m}$):
\begin{align}
    \checkS_m &= \sum_{k=1}^K\sum_{h=1}^H I^k_h[\VV\checkV_{k, h+1}^{2^m}](s_h^k, a_h^k), \label{eq:def:checksm} \\
    \tildeS_m &= \sum_{k=1}^K\sum_{h=1}^H I^k_h[\VV\tildeV_{k, h+1}^{2^m}](s_h^k, a_h^k), \label{eq:def:tildesm} \\
    S_m &= \sum_{k=1}^K\sum_{h=1}^H I^k_h[\VV V_{k, h+1}^{2^m}](s_h^k, a_h^k), \label{eq:def:sm} \\
    Q_m &= \sum_{k=1}^K\sum_{h=1}^H [\VV (V^*_{h+1})^{2^m}](s^k_h,a^k_h). \label{eq:def:qm}
\end{align}
We further use $\cQ_0^*$ to denote the maximum of $Q_m$:
\begin{equation}
\label{eq:def:cq}
    \cQ_0^*=\max_{m\in\seq{M}}Q_m.
\end{equation}
Then, for $2^m$-th order value functions ($\checkV_{k, h+1}^{2^m}, \tildeV_{k, h+1}^{2^m}, V_{k, h+1}^{2^m}$), we use $\checkA_m,\tildeA_m,A_m$ to denote the summation of stochastic transition noise as follows:
\begin{align}
    \checkA_m &= \left|\sum_{k = 1}^K \sum_{h=1}^H I^k_h \sbr{[\PP \checkV_{k,h+1}^{2^m}](s_h^k, a_h^k) - \checkV_{k,h+1}^{2^m}(s_{h+1}^k)}\right|, \label{eq:def:checkam} \\
    \tildeA_m &= \left|\sum_{k = 1}^K \sum_{h=1}^H I^k_h \sbr{[\PP \tildeV_{k,h+1}^{2^m}](s_h^k, a_h^k) - \tildeV_{k,h+1}^{2^m}(s_{h+1}^k)}\right|, \label{eq:def:tildeam} \\
    A_m &= \left|\sum_{k=1}^K \sum_{h=1}^H I^k_h \sbr{[\PP V_{k,h+1}^{2^m}](s_h^k, a_h^k) - V_{k,h+1}^{2^m}(s_{h+1}^k)}\right|. \label{eq:def:am}
\end{align}
Next, we use $V_0$ to represent the total rewards and $V^*_1$ for the average optimal value functions over $K$ episodes:
\begin{align}
    V_0 &= \sum_{k=1}^K\sum_{h=1}^H I^k_h r(s^k_h,a^k_h), \label{eq:def:v0} \\
    V^*_1 &= \frac{1}{K}\sum_{k=1}^K V^*_1(s^k_1). \label{eq:def:vstar}
\end{align}
Finally, we use the $R_m$ to denote the summation of bonuses:
\begin{align}
    R_m = \sum_{k=1}^K\sum_{h=1}^H I^k_h\min\{1,\hatbeta_k\cD_\cF(z_{k,h,m};D_{k,m})\}. \label{eq:def:rm}
\end{align}
Now, we introduce the following lemmas to build the connection between these quantities.

\begin{lemma}
\label{lem:G}
We have
\begin{equation}
\label{eq:G}
    G \leq (M+1)\dim_\cF.
\end{equation}
\end{lemma}
\begin{proof}
See Appendix~\ref{proof:G} for a detailed proof.
\end{proof}

\begin{lemma}
\label{lem:SA}
On event $\cE$, we have for all $m \in \seq{M}$,
\begin{align} 
    \checkS_m &\le \checkA_{m + 1} + G + 2^{m + 1}\cdot (2R_0), \label{eq:checkSA} \\
    \tildeS_m &\le \tildeA_{m + 1} + G + 2^{m + 1}\cdot (2R_0 + \checkA_0), \label{eq:tildeSA}\\
    S_m &\le A_{m + 1} + G + 2^{m + 1}\cdot (V_0 + 2R_0). \label{eq:SA}
\end{align}
\end{lemma}
\begin{proof}
See Appendix~\ref{proof:SA} for a detailed proof.
\end{proof}

\begin{lemma}
\label{lem:AS}
With probability at least $1 - 3(M+1)\delta$, we have for all $m \in \seq{M}$,
\begin{align}
    \checkA_m \le \sqrt{\zeta\checkS_m} + \zeta, \label{eq:checkAS}\\
    \tildeA_m \le \sqrt{\zeta\checkS_m} + \zeta, \label{eq:tildeAS}\\
    A_m \le \sqrt{\zeta S_m} + \zeta, \label{eq:AS}
\end{align}
where $\zeta=8\log(2(\log(KH)+2)/\delta)$. We denote the corresponding event by $\cA$.
\end{lemma}
\begin{proof}
See Appendix~\ref{proof:AS} for a detailed proof.
\end{proof}

\begin{lemma}
\label{lem:checkA0}
On event $\cE\cup\cA$, we have for all $m\in\seq{M}$,
\begin{equation}
\label{eq:checkA}
    \checkA_m \le \sqrt{\zeta}\cdot\sqrt{\checkA_{m + 1} + G + 2^{m + 1}\cdot (2R_0)} + \zeta.
\end{equation}
And
\begin{align}
    \checkA_0 &\le 2\sqrt{2\zeta R_0}+3\sqrt{\zeta G} + 7\zeta, \label{eq:checkA0} \\
    \checkA_1 &\le 4\sqrt{\zeta R_0}+3\sqrt{\zeta G} + 7\zeta. \label{eq:checkA1}
\end{align}
\end{lemma}
\begin{proof}
See Appendix~\ref{proof:checkA0} for a detailed proof.
\end{proof}

\begin{lemma}
\label{lem:tildeA0}
On event $\cE\cup\cA$, we have
\begin{equation}
\label{eq:tildeA0}
    \tildeA_0 \le 4\sqrt{2\zeta R_0}+6\sqrt{\zeta G} + 15\zeta.
\end{equation}
\end{lemma}
\begin{proof}
See Appendix~\ref{proof:tildeA0} for a detailed proof.
\end{proof}

\begin{lemma}
\label{lem:V0}
We have
\begin{equation}
\label{eq:V0}
    V_0 \le V^*_1 K + A_0 + \checkA_0. 
\end{equation}
\end{lemma}
\begin{proof}
See Appendix~\ref{proof:V0} for a detailed proof.
\end{proof}

\begin{lemma}[Formal version of Lemma~\ref{lem:higher-expansion}]
\label{lem:RS}
Let $\gamma^2\le\log\cN_\cF$. On event $\cE$, we have for all $m \in \seq{M-1}$,
\begin{equation}
\label{eq:RS}
    R_m \le 4\hat\beta_K\sqrt{\dim_\cF}\cdot\sqrt{S_m + 2R_{m+1} + KH\sigmamin^2} + 24\hatbeta_K^2\dim_\cF.
\end{equation}
\end{lemma}
\begin{proof}
See Appendix~\ref{proof:RS} for a detailed proof.
\end{proof}

\begin{lemma}[Formal version of Lemma~\ref{lem:higher-expansion-instance}]
\label{lem:RcheckA0}
On event $\cE\cup\cA$, we have
\begin{align}
    R_0 + \checkA_0 \le 32\hatbeta_K\sqrt{\dim_\cF R_0} + 15\hatbeta_K\sqrt{2\dim_\cF}\cdot\sqrt{\cQ_0^*+G+KH\sigmamin^2}+587\hatbeta_K^2\dim_\cF. \label{eq:RcheckA0}
\end{align}
\end{lemma}
\begin{proof}
See Appendix~\ref{proof:RcheckA0} for a detailed proof.
\end{proof}

\begin{lemma}
\label{lem:RA0}
On event $\cE\cup\cA$, we have
\begin{equation}
\label{eq:RA0}
    R_0 + A_0 \le 32\hatbeta_K\sqrt{\dim_\cF}\cdot\sqrt{V^*_1K+\checkA_0} + 30\hatbeta_K\sqrt{\dim_\cF}\cdot\sqrt{G+KH\sigmamin^2}+1174\hatbeta_K^2\dim_\cF.
\end{equation}
\end{lemma}
\begin{proof}
See Appendix~\ref{proof:RA0} for a detailed proof.
\end{proof}

\subsection{Regret Analysis}
\label{proof:regret}
\begin{proof}[Proof of Theorem~\ref{thm:mdp-regret}]
On event $\cE\cup\cA$, which holds with probability at least $1-4(M+1)\delta$ by Lemma~\ref{lem:event-concentration}, \ref{lem:AS} and a union bound, we have all lemmas in this section hold. By the optimism implied by Lemma~\ref{lem:optimism}, we have
\begin{equation}
\label{eq:regret-opt}
    \mathrm{Regret}(K) = \sum_{k=1}^K \sbr{V^*_1(s^k_1)-V^{\pi^k}_1(s^k_1)} \le \sum_{k=1}^K \sbr{V_{k,1}(s^k_1)-V^{\pi^k}_1(s^k_1)}.
\end{equation}
We further use Lemma~\ref{lem:regret-gap} to bound the regret with the higher-order quantities defined in Section~\ref{proof:higher-order}.
\begin{lemma}
\label{lem:regret-gap}
On event $\cE$, we have
\begin{equation}
\label{eq:regret-gap}
    \sum_{k=1}^K \sbr{V_{k,1}(s^k_1)-V^{\pi^k}(s^k_1)} \le 2R_0 + \checkA_0 + \tildeA_0 + G.
\end{equation}
\end{lemma}
\begin{proof}
See Appendix~\ref{proof:regret-gap} for a detailed proof.
\end{proof}
On one hand, we have
\begin{align}
    &2R_0 + \checkA_0 + \tildeA_0 + G \notag \\
    \overset{(a)}{\le}{} &64\hatbeta_K\sqrt{\dim_\cF R_0} + 30\hatbeta_K\sqrt{2\dim_\cF}\cdot\sqrt{\cQ_0^*+G+KH\sigmamin^2}+1174\hatbeta_K^2\dim_\cF \notag \\
    &\quad+ 4\sqrt{2\zeta R_0}+6\sqrt{\zeta G} + G \notag \\
    \overset{(b)}{=}{} &\tilde{O}\rbr{\hatbeta_K\sqrt{\dim_\cF}\cdot\sqrt{\cQ_0^*+KH\sigmamin^2} + \hatbeta_K^2\dim_\cF}, \label{eq:regret-variance}
\end{align}
where $(a)$ holds due to \eqref{eq:RcheckA0} and \eqref{eq:tildeA0}, while we utilize $\zeta\le\hatbeta_K^2\dim_\cF$ (which is trivial since $\zeta$ only consists of logarithmic terms), $x \le a\sqrt{x}+b$ implies $x\le a^2+2b$ for any $x\ge0$ and \eqref{eq:G} for $(b)$.
On the other hand, we have
\begin{align}
    &2R_0 + \checkA_0 + \tildeA_0 + G \notag \\
    \overset{(a)}{\le}{} &64\hatbeta_K\sqrt{\dim_\cF}\cdot\sqrt{V^*_1 K+\checkA_0} + 60\hatbeta_K\sqrt{\dim_\cF}\cdot\sqrt{G+KH\sigmamin^2}+2348\hatbeta_K^2\dim_\cF \notag \\
    &\quad+ 2\sqrt{2\zeta R_0}+3\sqrt{\zeta G}+7\zeta + 4\sqrt{2\zeta R_0}+6\sqrt{\zeta G} + G \notag \\
    \overset{(b)}{=}{} &\tilde{O}\rbr{\hatbeta_K\sqrt{\dim_\cF}\cdot\sqrt{V^*_1 K+KH\sigmamin^2} + \hatbeta_K^2\dim_\cF}, \label{eq:regret-first-order}
\end{align}
where $(a)$ holds due to \eqref{eq:RA0}, \eqref{eq:checkA0} and \eqref{eq:tildeA0}, while we utilize $\zeta\le\hatbeta_K^2\dim_\cF$, $x \le a\sqrt{x}+b$ implies $x\le a^2+2b$ for any $x\ge0$ and \eqref{eq:G} in $(b)$.
Combining \eqref{eq:regret-opt}, \eqref{eq:regret-gap}, \eqref{eq:regret-variance} and \eqref{eq:regret-first-order}, we have
\[
    \mathrm{Regret}(K) = \tilde{O}\rbr{\hatbeta_K\sqrt{\dim_\cF}\cdot\sqrt{\min\{\cQ_0^*,V^*_1 K\}+KH\sigmamin^2} + \hatbeta_K^2\dim_\cF}.
\]
Recall that according to \eqref{eq:hatbeta},
\[
    \hatbeta_K = 3\sqrt{\iota_K} + 2\frac{\iota_K}{\gamma^2} + \sqrt{\lambda} + \sqrt{6kH\epsilon/\sigmamin^2}
\]
with $\iota_K=\tilde{O}(\log\cN_\cF)$.
Moreover, setting $\lambda=\log\cN_\cF,\sigmamin^2=\sqrt{\dim_\cF \log\cN_\cF}/(KH),\gamma^2=\sqrt{\log\cN_\cF},\epsilon=\log\cN_\cF\sigmamin^2/(KH)$, we have $\hatbeta_K=\tilde{O}(\sqrt{\log\cN_\cF})$ and $KH\sigmamin^2=\sqrt{\dim_\cF\log\cN_\cF}$, which yields a high-probability regret bound
\[
    \mathrm{Regret}(K) = \tilde{O}\rbr{\sqrt{\dim_\cF\log\cN_\cF}\cdot\sqrt{\min\{\cQ_0^*,V^*_1 K\}} + \dim_\cF\log\cN_\cF}.
\]
Notice $\cQ^*=\min\{\cQ_0^*,V^*_1 K\}$, then the proof is completed.
\end{proof}

\ifarxiv
\section{Missing Proofs in Section~\ref{proof:mdp}}
\else
\section{MISSING PROOFS IN SECTION~\ref{proof:mdp}}
\fi
\label{proof:proof:mdp}

\subsection{Proof of Lemma~\ref{lem:event-bonus}}
\label{proof:event-bonus}
\begin{proof}[Proof of Lemma~\ref{lem:event-bonus}]
Recall $D_{k,m}=\{z_{i,h,m},\barsigma_{i,h,m}\}_{(i,h)\in[k-1]\times[H]}$ and $D_{k,h,m} = D_{k,m} \cup \{z_{k,j,m},\barsigma_{k,j,m}\}_{j\in[h]}$.
If $f_*\in\hat\cB_{k,m}$, it follows that
\begin{align*}
    &|\hatf_{k,m}(z_{k,h,m}) - f_*(z_{k,h,m})| \\
    \overset{(a)}{\le}{} &\cD_\cF(z_{k,h,m};D_{k,m}) \sqrt{\sum_{i=1}^{k-1} \sum_{h=1}^H \frac{1}{\barsigma_{i,h,m}^2} (\hatf_{k,h,m}(z_{i,h,m})-f_*(z_{i,h,m}))^2 + \lambda}\\
    \overset{(b)}{\le}{} &(\hatbeta_{k}+\sqrt{\lambda}) \cD_\cF(z_{k,h,m};D_{k,m})\\
    \overset{(c)}{\approx}{} &\hatbeta_{k}\cD_\cF(z_{k,h,m};D_{k,m}) ,
\end{align*}
where $(a)$ holds due to the definition of $\cD_\cF$ in Definition~\ref{def:general-eluder-RL}, $(b)$ holds due to $\sqrt{a+b}\le\sqrt{a}+\sqrt{b}$ for any $a,b\ge 0$, $(c)$ holds due to $\sqrt{\lambda}= O(\hatbeta_{k})$, therefore we omit the term $\sqrt{\lambda}$ in the proof for simplicity.
Furthermore, we have
\begin{align*}
    &V_{k,h}(s^k_h) - r_h(s^k_h,a^k_h) - [\PP V_{k,h+1}](s^k_h,a^k_h) \\
    ={} &\min\{1, r_h(s^k_h,a^k_h) + \hatf_{k,m}(z_{k,h,0}) + \hatbeta_k \cD_\cF(z_{k,h,0};D_{k,0})\} - r_h(s^k_h,a^k_h) - f_*(z_{k,h,0}) \\
    \le{} &\min\{1,2\hatbeta_k \cD_\cF(z_{k,h,0};D_{k,0})\}.
\end{align*}
And
\begin{align*}
    &|[\bar{\VV}_{k,h}V_{k,h+1}^{2^m}](s^k_h,a^k_h)-[\VV_{k,h}V_{k,h+1}^{2^m}](s^k_h,a^k_h)| \\
    ={} &|[f_{k,m+1}(z_{k,h,m+1}) - f_{k,m}^2(z_{k,h,m})] - [f_*(z_{k,h,m+1}) - f_*^2(z_{k,h,m})]| \\
    \le{} &|f_{k,m+1}(z_{k,h,m+1})-f_*(z_{k,h,m+1})| + |f_{k,m}^2(z_{k,h,m}) - f_*^2(z_{k,h,m})| \\
    ={} &|f_{k,m+1}(z_{k,h,m+1})-f_*(z_{k,h,m+1})| + |f_{k,m}(z_{k,h,m}) + f_*(z_{k,h,m})|\cdot|f_{k,m}(z_{k,h,m}) - f_*(z_{k,h,m})| \\
    \le{} &\min\{1,\hatbeta_k \cD_\cF(z_{k,h,m+1};D_{k,m+1})\} + \min\{1,2\hatbeta_k \cD_\cF(z_{k,h,m};D_{k,m})\} \\
    ={} &E_{k,h,m},
\end{align*}
where the last inequality holds due to $f_{k,m}(z_{k,h,m}),f_*(z_{k,h,m})\in[0,1]$.
\end{proof}

\subsection{Proof of Lemma~\ref{lem:event-concentration}}
\label{proof:event-concentration}
\begin{proof}[Proof of Lemma~\ref{lem:event-concentration}]
We will prove the statement by Theorem~\ref{thm:wls} and induction.
For each $m\in\seq{M-1}$, denote $\xi_{k,h,m}:=f_*(z_{k,h,m}) + \ind\{f_*\in\hat\cB_{k,m}\cap\hat\cB_{k,m+1}\}[y_{k,h,m}-f_*(z_{k,h,m})]$, and $\xi_{k,h,M}:=y_{k,h,M}$. Then we have for all $m\in\seq{M}$, $\xi_{k,h,m}\in\cG_{k,h}$, $\EE[\xi_{k,h,m}|\cG_{k,h}]=f_*(z_{k,h,m})$. Recall the definition of $\barsigma_{k,h,m}$ in \eqref{line:barsigma} of Algorithm~\ref{algo:variance}, for $m\in\seq{M-1}$, we have
\begin{align*}
    \barsigma_{k,h,m}^{-2}\Var[\xi_{k,h,m}|\cG_{k,h,m}] &= \barsigma_{k,h,m}^{-2} \ind\{f_*\in\hat\cB_{k,m}\cap\hat\cB_{k,m+1}\} \Var[y_{k,h,m}|\cG_{k,h,m}] \\
    &= \barsigma_{k,h,m}^{-2} \ind\{f_*\in\hat\cB_{k,m}\cap\hat\cB_{k,m+1}\} [\VV V_{k,h+1}^{2^m}](s^k_h,a^k_h) \\
    &= \barsigma_{k,h,m}^{-2} \ind\{f_*\in\hat\cB_{k,m}\cap\hat\cB_{k,m+1}\} [[\PP V_{k,h+1}^{2^{m+1}}](s^k_h,a^k_h) - [\PP V_{k,h+1}^{2^m}]^2(s^k_h,a^k_h)] \\
    &= \barsigma_{k,h,m}^{-2} \ind\{f_*\in\hat\cB_{k,m}\cap\hat\cB_{k,m+1}\} [f_*(z_{k,h,m+1}) - f_*^2(z_{k,h,m})] \\
    &\le \barsigma_{k,h,m}^{-2} [f_{k,m+1}(z_{k,h,m+1}) - f_{k,m}^2(z_{k,h,m}) + \min\{1,\hatbeta_{k}\cD_\cF(z_{k,h,m+1};D_{k,m+1})\\
    &\qquad \min\{1,2\hatbeta_{k}\cD_\cF(z_{k,h,m};D_{k,m})\}\}] \\
    &= \barsigma_{k,h,m}^{-2} [[\bar{\VV}_{k,m}V_{k, h+1}^{2^m}](s_h^k, a_h^k) + E_{k,h,m}] \\
    &\le 1,
\end{align*}
where the first inequality holds due to Lemma~\ref{lem:event-bonus} and $f_{k,m}(z_{k,h,m})\in[0,1]$. And
\[
    \barsigma_{k,h,M}^{-2}\Var[\xi_{k,h,M}|\cG_{k,h,m}]\le\barsigma_{k,h,M}^{-2}\le1.
\]
Furthermore, for all $m\in\seq{M}$, we have
\[
    \barsigma_{k,h,m}^{-2} \cD_\cF^2(z_{k,h,m};D_{k,m}) \le 1/\gamma^2.
\]
For each $m\in\seq{M}$, we define
\begin{equation}
\label{eq:tildef}
    \tilde{f}_{k+1,m} = \argmin_{f\in\cF} \sum_{i=1}^k\sum_{h=1}^H (f(z_{i,h,m}) - \xi_{i,h,m})^2.
\end{equation}
Applying Theorem~\ref{thm:wls} using $X_t = z_{k,h,m}$, $Y_t = \xi_{k,h,m}$, $w_t=1/\barsigma_{k,h,m}$, together with a union bound, with probability at least $1-(M+1)\delta$, we have for all $m\in\seq{M},k\in[K]$,
\begin{equation}
\label{eq:concentration-proxy}
    \sum_{i=1}^{k-1} \sum_{h=1}^H \frac{1}{\barsigma_{i,h,m}^2} (\tilde{f}_{k,h,m}(z_{i,h,m})-f_*(z_{i,h,m}))^2
    \le 3\sqrt{\iota_k} + 2\frac{\iota_k}{\gamma^2} + \sqrt{\lambda} + \sqrt{6kH\epsilon/\sigmamin^2}
    = \hatbeta_k^2,
\end{equation}
where $\iota_k = 16\log\frac{2\cN_\cF(\epsilon)k^2H^2 (\log(kH/\sigmamin^2)+2) (\log(1/\sigmamin^2)+2)}{\delta}$.

We continue the proof by induction over $(k,m)\in[K]\times\seq{M}$.
First, for $k=1$, $m\in\seq{M}$, the result holds trivially.

Next, for $k\in[K],m=M$, we have $\xi_{k,h,M}=y_{k,h,M}$, therefore $\tilde{f}_{k,M}=\hatf_{k,M}$ by \eqref{eq:tildef}, which implies $\hatf_{k,M}\in\hat\cB_{k,M}$ by \eqref{eq:concentration-proxy}.

Last, for all $k\in[k],m\in\seq{M-1}$, we have the following observations:
\begin{equation}
\label{eq:event-reduction}
\begin{aligned}
    &f_* \in \hat\cB_{k,m}\cap\hat\cB_{k,m+1} \\
    \Longrightarrow{} &\xi_{k,h,m} = y_{k,h,m}\text{ for all }h\in[H]\\
    \overset{\eqref{eq:tildef}}{\Longrightarrow}{} &\tilde{f}_{k+1,m} = \hatf_{k+1,m} \\
    \overset{\eqref{eq:concentration-proxy}}{\Longrightarrow}{} &\hatf_{k+1,m}\in\hat\cB_{k+1,m}.
\end{aligned}
\end{equation}
For any $k\in K$, we assume $\hatf_{k',m}\in\hat\cB_{k',m}$ for all $k'\le k,m\in\seq{M}$. Notice $\hatf_{k+1,M}\in\hat\cB_{k+1,M}$, using \eqref{eq:event-reduction}, we have $\hatf_{k+1,m}\in\hat\cB_{k+1,m}$ for all $m\in\seq{M}$. Then the proof is completed by induction. 
\end{proof}

\subsection{Proof of Lemma~\ref{lem:optimism}}
\label{proof:optimism}
\begin{proof}[Proof of Lemma~\ref{lem:optimism}]
We prove the optimism by induction. When $h=H+1$, we have $V_{k,H+1}(\cdot)=V^*_{H+1}(\cdot)=0$, and the result holds trivially.
We assume the statement is true for all $h+1$, and prove the case of $h$.
For any $(s,a)$, if $Q_{k,h}(s,a)=1$, then $Q_{k,h}(s,a)=1\ge Q^*_h(s,a)$. Otherwise, we have
\begin{align*}
    Q_{k,h}(s,a) - Q^*_h(s,a) &= f_{k,0}(s,a,V_{k,h+1}) - f_*(s,a,V^*_{h+1}) + \hatbeta_k \cD_\cF(s,a,V_{k,h+1};D_{k,0}) \\
    &\ge f_{k,0}(s,a,V_{k,h+1}) - f_*(s,a,V_{k,h+1}) + \hatbeta_k \cD_\cF(s,a,V_{k,h+1};D_{k,0}) \\
    &\ge 0,
\end{align*}
where the first inequality holds due to $V_{k,h+1}(\cdot)\ge V^*_{h+1}(\cdot)$ and the second holds due to Lemma~\ref{lem:event-bonus}. That is, we have $Q_{k,h}(\cdot,\cdot)\ge Q^*_h(\cdot,\cdot)$ and therefore $V_{k,h}(\cdot)\ge V^*_h(\cdot)$. Then the proof is completed by induction.
\end{proof}

\subsection{Proof of Lemma~\ref{lem:G}}
\label{proof:G}
\begin{proof}[Proof of Lemma~\ref{lem:G}]
Recall the definition of $I_h^k$, we have
\[
    1-I_{H+1}^k = 1 \Leftrightarrow \exists m \in \seq{M},\ \sum_{h=1}^{H} \frac{1}{\barsigma_{k,h,m}^2}\cD_\cF^2(z_{k,h,m};D_{k,h-1,m})>1.
\]
Let $\cD_m$ denote the indices $k$ such that 
\begin{align}
    \cD_m:=\bigg\{k \in [K]: \sum_{h=1}^{H} \frac{1}{\barsigma_{k,h,m}^2}\cD_\cF^2(z_{k,h,m};D_{k,h-1,m})>1 \bigg\}.\notag
\end{align}
Then we have $G \leq |\bigcup_{m=0}^M \cD_m| \leq \sum_{m=0}^M|\cD_m|$. For any $m\in\seq{M}$, we have
\begin{align*}
    |\cD_m| &\le \sum_{k=1}^K\min\cbr{1,\sum_{h=1}^{H} \frac{1}{\barsigma_{k,h,m}^2}\cD_\cF^2(z_{k,h,m};D_{k,h-1,m})}\\
    &\le \sum_{k=1}^K\sum_{h=1}^{H}\min\cbr{1,\frac{1}{\barsigma_{k,h,m}^2}\cD_\cF^2(z_{k,h,m};D_{k,h-1,m})}\\
    &\le \dim_\cF.
\end{align*}
Taking the summation over $m\in\seq{M}$ gives the upper bound of $G$.
\end{proof}

\subsection{Proof of Lemma~\ref{lem:SA}}
\label{proof:SA}
\begin{proof}[Proof of Lemma~\ref{lem:SA}]
We are to bound $\checkS_m,\tildeS_m$ and $S_m$ separately with similar arguments.

\paragraph{Bound $\checkS_m$}
Recall the definition of $\checkS_m$ in \eqref{eq:def:checksm}, we have
\begin{align}
    \checkS_m & = \sum_{k = 1}^K \sum_{h = 1}^H I^k_h [\VV \checkV_{k, h + 1}^{2^m}](s_h^k, a_h^k) \notag \\
    &= \sum_{k = 1}^K \sum_{h = 1}^H I^k_h \left[[\PP \checkV_{k, h + 1}^{2^{m + 1}}](s_h^k, a_h^k) - [\PP \checkV_{k, h + 1}^{2^m}]^2(s_h^k, a_h^k)^2\right] \notag \\
    &= \sum_{k = 1}^K \sum_{h = 1}^H I^k_h \left[[\PP \checkV_{k, h + 1}^{2^{m + 1}}](s_h^k, a_h^k) - \checkV_{k, h + 1}^{2^{m + 1}}(s_{h + 1}^k)\right] + \sum_{k = 1}^K \sum_{h = 1}^H I^k_h \left[\checkV_{k, h}^{2^{m + 1}}(s_{h}^k) - [\PP \checkV_{k, h + 1}^{2^m}]^2(s_h^k, a_h^k)\right] \notag \\
    &\qquad + \sum_{k = 1}^K \sum_{h = 1}^H I^k_h \rbr{\checkV_{k, h + 1}^{2^{m + 1}}(s_{h + 1}^k) - \checkV_{k, h}^{2^{m + 1}}(s_{h}^k)} \notag \\
    &\le \checkA_{m+1} + \sum_{k = 1}^K \sum_{h = 1}^H I^k_h \left[\checkV_{k, h}^{2^{m + 1}}(s_{h}^k) - [\PP \checkV_{k, h + 1}^{2^m}]^2(s_h^k, a_h^k)\right] + \sum_{k=1}^K(1-I^k_{h_k})\checkV_{k, h_k}^{2^{m + 1}}(s_{h_k}^k) \notag \\
    &\le \checkA_{m+1} + G + \sum_{k = 1}^K \sum_{h = 1}^H I^k_h \left[\checkV_{k, h}^{2^{m + 1}}(s_{h}^k) - [\PP \checkV_{k, h + 1}^{2^m}]^2(s_h^k, a_h^k)\right], \label{eq:checkSA-sum}
\end{align}
where $h_k$ is defined in \eqref{eq:def:h_k} and the last inequality holds since $|\checkV_{k, h}(s_{h}^k)|\le1$ and $I^k_{h}$ is monotonically decreasing.
For the third term in \eqref{eq:checkSA-sum}, we have
\begin{align} 
    &\sum_{k = 1}^K \sum_{h = 1}^H I^k_h \left[\checkV_{k, h}^{2^{m + 1}}(s_{h}^k) - [\PP \checkV_{k, h + 1}^{2^m}]^2(s_h^k, a_h^k)\right] \notag \\
    \overset{(a)}{\le}{}& \sum_{k = 1}^K \sum_{h = 1}^H I^k_h \left[\checkV_{k, h}^{2^{m + 1}}(s_{h}^k) - [\PP \checkV_{k, h + 1}]^{2^{m + 1}}(s_h^k, a_h^k)\right] \notag \\
    ={}& \sum_{k = 1}^K \sum_{h = 1}^H I^k_h \sbr{\checkV_{k, h}(s_h^k) - [\PP \checkV_{k, h + 1}](s_h^k, a_h^k)} \prod_{i=0}^m \sbr{\checkV_{k, h}^{2^i}(s_h^k) + [\PP \checkV_{k, h + 1}]^{2^i}(s_h^k, a_h^k)} \notag \\
    \le{}& 2^{m + 1} \sum_{k = 1}^K \sum_{h = 1}^H I_h^k \max\cbr{\checkV_{k, h}(s_h^k) - [\PP \checkV_{k, h + 1}](s_h^k, a_h^k), 0} \notag \\
    \overset{(b)}{\le}{}& 2^{m + 1} \sum_{k = 1}^K \sum_{h = 1}^H I_h^k \max\cbr{V_{k,h}(s_h^k) - r(s_h^k, a_h^k) - [\PP V_{k,h+1}](s_h^k, a_h^k), 0} \notag \\
    \overset{(c)}{\le}{}& 2^{m + 1} \sum_{k = 1}^K \sum_{h = 1}^H I_h^k \cdot 2\min\{1,\hatbeta_k\cD_\cF(z_{k,h,0};D_{k,0})\} \notag \\
    ={}& 2^{m + 1} \cdot (2R_0), \label{eq:checkSA-exp}
\end{align}
where $(a)$ holds due to $\EE[X^2]\ge(\EE[X])^2$, $(b)$ holds due to the definition of $\checkV_{k,h}$ and $V^*_h(s^k_h) \ge r(s^k_h,a^k_h) + [\PP V^*_{h+1}](s^k_h,a^k_h)$, while $(c)$ is due to Lemma~\ref{lem:event-bonus}.
Substituting \eqref{eq:checkSA-exp} into \eqref{eq:checkSA-sum}, we have
\[
    \checkS_m \le \checkA_{m + 1} + G + 2^{m + 1}\cdot (2R_0).
\]

\paragraph{Bound $\tildeS_m$}
Recall the definition of $\tildeS_m$ in \eqref{eq:def:tildesm}, we have
\begin{align}
    \tildeS_m & = \sum_{k = 1}^K \sum_{h = 1}^H I^k_h [\VV \tildeV_{k, h + 1}^{2^m}](s_h^k, a_h^k) \notag \\
    &= \sum_{k = 1}^K \sum_{h = 1}^H I^k_h \left[[\PP \tildeV_{k, h + 1}^{2^{m + 1}}](s_h^k, a_h^k) - [\PP \tildeV_{k, h + 1}^{2^m}]^2(s_h^k, a_h^k)^2\right] \notag \\
    &= \sum_{k = 1}^K \sum_{h = 1}^H I^k_h \left[[\PP \tildeV_{k, h + 1}^{2^{m + 1}}](s_h^k, a_h^k) - \tildeV_{k, h + 1}^{2^{m + 1}}(s_{h + 1}^k)\right] + \sum_{k = 1}^K \sum_{h = 1}^H I^k_h \left[\tildeV_{k, h}^{2^{m + 1}}(s_{h}^k) - [\PP \tildeV_{k, h + 1}^{2^m}]^2(s_h^k, a_h^k)\right] \notag \\
    &\qquad + \sum_{k = 1}^K \sum_{h = 1}^H I^k_h \rbr{\tildeV_{k, h + 1}^{2^{m + 1}}(s_{h + 1}^k) - \tildeV_{k, h}^{2^{m + 1}}(s_{h}^k)} \notag \\
    &\le \tildeA_{m+1} + \sum_{k = 1}^K \sum_{h = 1}^H I^k_h \left[\tildeV_{k, h}^{2^{m + 1}}(s_{h}^k) - [\PP \tildeV_{k, h + 1}^{2^m}]^2(s_h^k, a_h^k)\right] + \sum_{k=1}^K (1-I^k_{h_k})\tildeV_{k, h_k}^{2^{m + 1}}(s_{h_k}^k) \notag \\
    &\le \tildeA_{m+1} + G + \sum_{k = 1}^K \sum_{h = 1}^H I^k_h \left[\tildeV_{k, h}^{2^{m + 1}}(s_{h}^k) - [\PP \tildeV_{k, h + 1}^{2^m}]^2(s_h^k, a_h^k)\right], \label{eq:tildeSA-sum}
\end{align}
where $h_k$ is defined in \eqref{eq:def:h_k} and the last inequality holds since $|\tildeV_{k, h}(s_{h}^k)|\le1$ and $I^k_{h}$ is monotonically decreasing.
For the third term in \eqref{eq:tildeSA-sum}, we have
\begin{align}
    &\sum_{k = 1}^K \sum_{h = 1}^H I^k_h \left[\tildeV_{k, h}^{2^{m + 1}}(s_{h}^k) - [\PP \tildeV_{k, h + 1}^{2^m}]^2(s_h^k, a_h^k)\right] \notag \\
    \overset{(a)}{\le}{}& \sum_{k = 1}^K \sum_{h = 1}^H I^k_h \left[\tildeV_{k, h}^{2^{m + 1}}(s_{h}^k) - [\PP \tildeV_{k, h + 1}]^{2^{m + 1}}(s_h^k, a_h^k)\right] \notag \\
    ={}& \sum_{k = 1}^K \sum_{h = 1}^H I^k_h \sbr{\tildeV_{k, h}(s_h^k) - [\PP \tildeV_{k, h + 1}](s_h^k, a_h^k)} \prod_{i=0}^m \sbr{\tildeV_{k, h}^{2^i}(s_h^k) + [\PP \tildeV_{k, h + 1}]^{2^i}(s_h^k, a_h^k)} \notag \\
    \le{}& 2^{m + 1} \sum_{k = 1}^K \sum_{h = 1}^H I_h^k \max\cbr{\tildeV_{k, h}(s_h^k) - [\PP \tildeV_{k, h + 1}](s_h^k, a_h^k), 0} \notag \\
    \overset{(b)}{=}{}& 2^{m + 1} \sum_{k = 1}^K \sum_{h = 1}^H I_h^k \max\cbr{V^*_h(s_h^k) - r(s_h^k, a_h^k) - [\PP V^*_{h+1}](s_h^k, a_h^k), 0} \notag \\
    \overset{(c)}{\le}{}& 2^{m + 1} \sum_{k = 1}^K \sum_{h = 1}^H I_h^k \max\cbr{V_{k,h}(s_h^k) - r(s_h^k, a_h^k) - [\PP V_{k,h+1}](s_h^k, a_h^k), 0} + |[\PP \checkV_{k,h+1}](s_h^k, a_h^k) - \checkV_{k,h+1}(s_{h+1}^k)| \notag \\
    \overset{(d)}{\le}{}& 2^{m + 1} \sum_{k = 1}^K \sum_{h = 1}^H I_h^k \sbr{2\min\{1,\hatbeta_k\cD_\cF(z_{k,h,0};D_{k,0})\} + |[\PP \checkV_{k,h+1}](s_h^k, a_h^k) - \checkV_{k,h+1}(s_{h+1}^k)|} \notag \\
    \le{}& 2^{m + 1} \cdot (2R_0 + \checkA_0), \label{eq:tildeSA-exp}
\end{align}
where $(a)$ holds due to $\EE[X^2]\ge(\EE[X])^2$, $(b)$ holds due to the definition of $\tildeV_{k,h}$ and $V^{\pi^k}_h(s^k_h) = r(s^k_h,a^k_h) + [\PP V^{\pi^k}_{h+1}](s^k_h,a^k_h)$, $(c)$ holds due to $V^*_h(s_h^k) \ge r(s_h^k, a_h^k) + [\PP V^*_{h+1}](s_h^k, a_h^k)$ and the definition of $\checkV_{k,h}$, while $(d)$ is due to Lemma~\ref{lem:event-bonus}.
Substituting \eqref{eq:tildeSA-exp} into \eqref{eq:tildeSA-sum}, we have
\[
    \tildeS_m \le \tildeA_{m + 1} + G + 2^{m + 1}\cdot (2R_0 + \checkA_0).
\]

\paragraph{Bound $S_m$}
Recall the definition of $S_m$ in \eqref{eq:def:sm}, we have
\begin{align}
    S_m &= \sum_{k = 1}^K \sum_{h = 1}^H I^k_h [\VV V_{k, h + 1}^{2^m}](s_h^k, a_h^k) \notag \\
    &= \sum_{k = 1}^K \sum_{h = 1}^H I^k_h \left[[\PP V_{k, h + 1}^{2^{m + 1}}](s_h^k, a_h^k) - [\PP V_{k, h + 1}^{2^m}]^2(s_h^k, a_h^k)\right] \notag \\
    &= \sum_{k = 1}^K \sum_{h = 1}^H I^k_h \left[[\PP V_{k, h + 1}^{2^{m + 1}}](s_h^k, a_h^k) - V_{k, h + 1}^{2^{m + 1}}(s_{h + 1}^k)\right] +  \sum_{k = 1}^K \sum_{h = 1}^H I^k_h \left[V_{k, h}^{2^{m + 1}}(s_{h}^k) - [\PP V_{k, h + 1}^{2^m}]^2(s_h^k, a_h^k)\right] \notag \\
    &\qquad + \sum_{k = 1}^K \sum_{h = 1}^H I^k_h \rbr{V_{k, h + 1}^{2^{m + 1}}(s_{h + 1}^k) - V_{k, h}^{2^{m + 1}}(s_{h}^k)} \notag \\
    &\le A_{m+1} + \sum_{k = 1}^K \sum_{h = 1}^H I^k_h \left[V_{k, h}^{2^{m + 1}}(s_{h}^k) - [\PP V_{k, h + 1}^{2^m}]^2(s_h^k, a_h^k)\right] + \sum_{k=1}^K (1-I^k_{h_k})V_{k, h_k}^{2^{m + 1}}(s_{h_k}^k) \notag \\
    &\le A_{m+1} + G + \sum_{k = 1}^K \sum_{h = 1}^H I^k_h \left[V_{k, h}^{2^{m + 1}}(s_{h}^k) - [\PP V_{k, h + 1}^{2^m}]^2(s_h^k, a_h^k)\right], \label{eq:SA-sum}
\end{align}
where $h_k$ is defined in \eqref{eq:def:h_k} and the last inequality holds since $|V_{k, h}(s_{h}^k)|\le1$ and $I^k_{h}$ is monotonically decreasing.
For the third term in \eqref{eq:SA-sum}, we have
\begin{align}
    &\sum_{k = 1}^K \sum_{h = 1}^H I^k_h \left[V_{k, h}^{2^{m + 1}}(s_{h}^k) - [\PP V_{k, h + 1}^{2^m}]^2(s_h^k, a_h^k)\right] \notag \\
    \overset{(a)}{\le}{}& \sum_{k = 1}^K \sum_{h = 1}^H I^k_h \left[V_{k, h}^{2^{m + 1}}(s_{h}^k) - [\PP V_{k, h + 1}]^{2^{m + 1}}(s_h^k, a_h^k)\right] \notag \\
    ={}& \sum_{k = 1}^K \sum_{h = 1}^H I^k_h \sbr{V_{k, h}(s_h^k) - [\PP V_{k, h + 1}](s_h^k, a_h^k)} \prod_{i=0}^m \sbr{V_{k, h}^{2^i}(s_h^k) + [\PP V_{k, h + 1}]^{2^i}(s_h^k, a_h^k)} \notag \\
    \le{}& 2^{m + 1} \sum_{k = 1}^K \sum_{h = 1}^H I_h^k \max\cbr{V_{k, h}(s_h^k) - [\PP V_{k, h + 1}](s_h^k, a_h^k), 0} \notag \\
    \overset{(b)}{\le}{}& 2^{m + 1} \sum_{k = 1}^K \sum_{h = 1}^H I_h^k \sbr{r(s^k_h,a^k_h) + 2\min\{1,\hatbeta_k\cD_\cF(z_{k,h,0};D_{k,0})\}} \notag \\
    \le{}& 2^{m + 1} \cdot (V_0 + 2R_0), \label{eq:SA-exp}
\end{align}
where $(a)$ holds due to $\EE[X^2]\ge(\EE[X])^2$ and $(b)$ holds due to Lemma~\ref{lem:event-bonus}.
Substituting \eqref{eq:SA-exp} into \eqref{eq:SA-sum}, we have
\[
    S_m \le A_{m + 1} + G + 2^{m + 1}\cdot(V_0 + 2R_0).
\]
\end{proof}

\subsection{Proof of Lemma~\ref{lem:AS}}
\label{proof:AS}
\begin{proof}[Proof of Lemma~\ref{lem:AS}]
Let $X_{k,h} = I^k_h\sbr{[\PP \checkV_{k,h+1}^{2^m}](s_h^k, a_h^k) - \checkV_{k,h+1}^{2^m}(s_{h+1}^k)}$, then we have $\EE[X_{k,h}|\cG_{k,h}]=0$, $|X_{k,h}|\le2$ and $\EE[X_{k,h}^2|\cG_{k,h}]=I^k_h[\VV \checkV_{k,h+1}^{2^m}](s^k_h,a^k_h)$. Therefore, for any $m\in\seq{M}$, applying variance-aware Freedman's inequality in Lemma \ref{lem:freedman-variance}, with probability at least $1 - \delta$, we have
\[
    \checkA_m \le \sqrt{\zeta\checkS_m} + \zeta.
\]
Thus, taking a union bound over $m \in \seq{M}$, with probability at least $1-(M+1)\delta$, \eqref{eq:checkAS} holds.
The proofs for \eqref{eq:tildeAS} and \eqref{eq:AS} follow the same arguments as \eqref{eq:checkAS}.
\end{proof}

\subsection{Proof of Lemma~\ref{lem:checkA0}}
\label{proof:checkA0}
\begin{proof}[Proof of Lemma~\ref{lem:checkA0}]
On event $\cE\cup\cA$, we have \eqref{eq:checkSA} and \eqref{eq:checkAS} hold by Lemma~\ref{lem:SA} and \ref{lem:AS}. Substituting the bound of $\checkS_m$ in \eqref{eq:checkSA} into \eqref{eq:checkAS}, we have for all $m\in\seq{M}$,
\[
    \checkA_m \le \sqrt{\zeta}\cdot\sqrt{\checkA_{m + 1} + G + 2^{m + 1}\cdot (2R_0)} + \zeta.
\]
And we have for all $m\in\seq{M}$, $\checkA_m\le2KH$. Then the result follows by Lemma~\ref{lem:exp}.
\end{proof}

\subsection{Proof of Lemma~\ref{lem:tildeA0}}
\label{proof:tildeA0}
\begin{proof}[Proof of Lemma~\ref{lem:tildeA0}]
On event $\cE\cup\cA$, we have \eqref{eq:tildeSA} and \eqref{eq:tildeAS} hold by Lemma~\ref{lem:SA} and \ref{lem:AS}. Substituting the bound of $\tildeS_m$ in \eqref{eq:tildeSA} into \eqref{eq:tildeAS}, we have for all $m\in\seq{M}$,
\[
    \tildeA_m \le \sqrt{\zeta}\cdot\sqrt{\tildeA_{m + 1} + G + 2^{m + 1}\cdot (2R_0 + \checkA_0)} + \zeta.
\]
And we have for all $m\in\seq{M}$, $\tildeA_m\le2KH$. Applying Lemma~\ref{lem:exp}, we have
\begin{align*}
    \tildeA_0 &\le 2\sqrt{\zeta(2R_0+\checkA_0)}+3\sqrt{\zeta G} + 7\zeta \\
    &\overset{(a)}{\le} 2\sqrt{2\zeta R_0}+2\sqrt{\zeta\checkA_0}+3\sqrt{\zeta G}+7\zeta \\
    &\overset{(b)}{\le} 2\sqrt{2\zeta R_0}+\zeta+\checkA_0+3\sqrt{\zeta G}+7\zeta \\
    &\overset{(c)}{\le} 4\sqrt{2\zeta R_0}+6\sqrt{\zeta G}+15\zeta,
\end{align*}
where $(a)$ holds due to $\sqrt{a+b}\le\sqrt{a}+\sqrt{b}$ for $a,b\ge0$, $(b)$ holds due to $2\sqrt{ab} \le a + b$ for $a,b\ge0$ and $(c)$ holds due to \eqref{eq:checkA0} in Lemma~\ref{lem:checkA0}.
\end{proof}

\subsection{Proof of Lemma~\ref{lem:V0}}
\label{proof:V0}
\begin{proof}[Proof of Lemma~\ref{lem:V0}]
Recall the definition of $V_0$ and $V^*_1$ in \eqref{eq:def:v0} and \eqref{eq:def:vstar}, we have
\begin{align*}
    V_0 - V^*_1 K &= \sum_{k=1}^K\sum_{h=1}^H I^k_h r(s^k_h,a^k_h) - \sum_{k=1}^K V^*_1(s^k_h) \\
    &\le \sum_{k=1}^K\sum_{h=1}^H I^k_h \sbr{r(s^k_h,a^k_h) - V^*_h(s^k_h) + V^*_{h+1}(s^k_{h+1})} \\
    &\le \sum_{k=1}^K\sum_{h=1}^H I^k_h \sbr{V^*_{h+1}(s^k_{h+1}) - [\PP_h V^*_{h+1}](s^k_h,a^k_h)} \\
    &\le A_0 + \checkA_0,
\end{align*}
where the second inequality holds due to $V^*_h(s^k_h) \ge r(s^k_h,a^k_h) + [\PP V^*_{h+1}](s^k_h,a^k_h)$.
\end{proof}

\subsection{Proof of Lemma~\ref{lem:RS}}
\label{proof:RS}
\begin{proof}[Proof of Lemma~\ref{lem:RS}]
First, for all $m\in\seq{M}$, $R_m\le KH$ holds trivially.
For $(k,h)$ where $I_h^k = 1$, we have $\sum_{j=1}^{h-1}\frac{1}{\barsigma_{k,j,m}^2}\cD_\cF^2(z_{k,j,m};D_{k,j-1,m})\le1$. By Lemma \ref{lem:cD-relation}, it follows that
\begin{align*}
    &\cD_\cF(z_{k,h,m};D_{k,m})\\
    \le{} & \exp\cbr{\frac{1}{2}\sum_{j=1}^{h-1}\frac{1}{\barsigma_{k,j,m}^2}\cD_\cF^2(z_{k,j,m};D_{k,j-1,m})}\cD_\cF(z_{k,h,m};D_{k,h-1,m})\\
    \le{} & 2\cD_\cF(z_{k,h,m};D_{k,h-1,m}).
\end{align*}
Then we have
\[
    R_m \le 2\sum_{k=1}^K\sum_{h=1}^H \min\{1,I^k_h\hatbeta_k\cD_\cF(z_{k,h,m};D_{k,h-1,m})\},
\]
which can be bounded by Lemma \ref{lem:sum-bonus}, with $\beta_t = I_h^k\hat\beta_k$, $\bar\sigma_t = \bar\sigma_{k,h,m}$ and $X_t=z_{k,h,m}$. We have
\begin{align*}
    &\sum_{k=1}^K\sum_{h=1}^H \min\{1,I^k_h\hatbeta_k\cD_\cF(z_{k,h,m};D_{k,h-1,m})\}\\
    \le{} &\dim_\cF + \hatbeta_K\gamma^2\dim_\cF + \hatbeta_K\sqrt{\dim_\cF}\cdot\sqrt{\sum_{k=1}^K\sum_{h=1}^HI_h^k\big[[\bar\VV_{k,m}V_{k, h+1}^{2^m}](s_h^k, a_h^k) + E_{k,h,m}\big] + KH\sigmamin^2}\\
    \le{} &\dim_\cF + \hatbeta_K\gamma^2\dim_\cF + \hatbeta_K\sqrt{\dim_\cF}\cdot\sqrt{\sum_{k=1}^K\sum_{h=1}^HI_h^k\big[[\VV V_{k, h+1}^{2^m}](s_h^k, a_h^k) + 2E_{k,h,m}\big] + KH\sigmamin^2},
\end{align*}
where the last inequality is due to Lemma~\ref{lem:event-bonus}. Note that
\begin{align*}
    \sum_{k=1}^K\sum_{h=1}^HI_h^kE_{k,h,m} &= \sum_{k=1}^K\sum_{h=1}^HI_h^k\min\Big\{1,2\hatbeta_k\cD_\cF(z_{k,h,m};D_{k,m})\Big\}\\
    &\qquad +\sum_{k=1}^K\sum_{h=1}^HI_h^k\min\Big\{1,\hatbeta_k\cD_\cF(z_{k,h,m+1};D_{k,m+1})\Big\}\\
    &\le 2R_m + R_{m+1},
\end{align*}
where the last inequality holds by the definition of $R_m$ in \eqref{eq:def:rm}. Thus, we have
\begin{align*}
    R_m &\le 2\dim_\cF + 2\hatbeta_K \gamma^2\dim_\cF + 2\hat\beta_K\sqrt{\dim_\cF}\cdot\sqrt{S_m + 4R_m + 2R_{m+1} + KH\sigmamin^2} \\
    &\overset{(a)}{\le} 16\hatbeta_K^2\dim_\cF + 4\dim_\cF + 4\hatbeta_K \gamma^2\dim_\cF + 4\hat\beta_K\sqrt{\dim_\cF}\cdot\sqrt{S_m + 2R_{m+1} + KH\sigmamin^2} \\
    &\overset{(b)}{\le} 4\hat\beta_K\sqrt{\dim_\cF}\cdot\sqrt{S_m + 2R_{m+1} + KH\sigmamin^2} + 24\hatbeta_K^2\dim_\cF,
\end{align*}
where $(a)$ holds since $\sqrt{a+b}\le\sqrt{a}+\sqrt{b}$ for any $a,b\ge0$ and $x\le a\sqrt{x}+b$ implies $x\le a^2+2b$ for any $x\ge 0$, while $(b)$ holds due to $\hatbeta_K,\dim_\cF\ge1$ and $\gamma^2\le\hatbeta_K$.
\end{proof}

\subsection{Proof of Lemma~\ref{lem:RcheckA0}}
\begin{proof}[Proof of Lemma~\ref{lem:RcheckA0}]
\label{proof:RcheckA0}
On event $\cE\cup\cA$, we have \eqref{eq:RS} holds by Lemma~\ref{lem:RS}. And we have for all $m\in\seq{M}$,
\begin{equation}
\label{eq:SQcheckS}
    S_m \le 2Q_m + 2\checkS_m,
\end{equation}
where the inequality holds due to $\Var[X+Y]\le2\Var[X]+2\Var[Y]$.
Substituting the bound of $S_m$ in \eqref{eq:SQcheckS} into \eqref{eq:RS}, we have for all $m\in\seq{M-1}$,
\begin{equation}
\label{eq:RQcheckS}
    R_m \le 4\hat\beta_K\sqrt{\dim_\cF}\cdot\sqrt{2Q_m + 2\checkS_m + 2R_{m+1} + KH\sigmamin^2} + 24\hatbeta_K^2\dim_\cF.
\end{equation}
Next, on event $\cE\cup\cA$, we have \eqref{eq:checkSA} holds by Lemma~\ref{lem:SA}. Substituting the bound of $\checkS_m$ in \eqref{eq:checkSA} into \eqref{eq:RQcheckS}, we have
\begin{align}
    R_m &\le 4\hat\beta_K\sqrt{\dim_\cF}\cdot\sqrt{2R_{m+1} + 2\checkA_{m+1} + 2Q_m + 2G + KH\sigmamin^2 + 2^{m + 1}\cdot(4R_0)} + 24\hatbeta_K^2\dim_\cF \notag \\
    &\le 4\hat\beta_K\sqrt{2\dim_\cF}\cdot\sqrt{R_{m+1} + \checkA_{m+1} + 2^{m + 1}\cdot(2R_0)} \notag\\
    &\qquad+ 4\hat\beta_K\sqrt{2\dim_\cF}\cdot\sqrt{\cQ + G + KH\sigmamin^2} + 24\hatbeta_K^2\dim_\cF, \label{eq:RRcheckA}
\end{align}
where the last inequality holds due to the inequality that $\sqrt{a+b}\le\sqrt{a}+\sqrt{b}$ for any $a,b\ge0$, and the definition of $\cQ$ in \eqref{eq:def:cq}.
Then, on event $\cE\cup\cA$, we have \eqref{eq:checkA} holds by Lemma~\ref{lem:checkA0}. Add \eqref{eq:checkA} to \eqref{eq:RRcheckA}, using $\sqrt{a}+\sqrt{b}\le\sqrt{2}\cdot\sqrt{a+b}$ for any $a,b\ge0$ and assume $\zeta\le\hatbeta_K^2\dim_\cF$ (which is trivial since $\zeta$ only consists of logarithmic terms), we have for all $m\in\seq{M-1}$,
\begin{align}
    R_m + \checkA_m &\le 8\hat\beta_K\sqrt{\dim_\cF}\cdot\sqrt{R_{m+1}+2\checkA_{m+1}+2^{m + 1}\cdot(4R_0)} \notag \\
    &\qquad + 5\hat\beta_K\sqrt{2\dim_\cF}\cdot\sqrt{\cQ + G + KH\sigmamin^2} + 25\hatbeta_K^2\dim_\cF \notag \\
    &\le 8\hat\beta_K\sqrt{2\dim_\cF}\cdot\sqrt{R_{m+1}+\checkA_{m+1}+2^{m + 1}\cdot(2R_0)} \notag \\
    &\qquad + 5\hat\beta_K\sqrt{2\dim_\cF}\cdot\sqrt{\cQ + G + KH\sigmamin^2} + 25\hatbeta_K^2\dim_\cF. \label{eq:RcheckA}
\end{align}
And for all $m\in\seq{M}$, $R_m + \checkA_m \le 3HK$. Then the result follows by Lemma~\ref{lem:exp}.
\end{proof}

\subsection{Proof of Lemma~\ref{lem:RA0}}
\label{proof:RA0}
\begin{proof}[Proof of Lemma~\ref{lem:RA0}]
On event $\cE\cup\cA$, we have \eqref{eq:SA}, \eqref{eq:AS} and \eqref{eq:RS} holds by Lemma~\ref{lem:SA}, \ref{lem:AS} and \ref{lem:RS}. Substituting the bound of $S_m$ in \eqref{eq:SA} into \eqref{eq:AS} and \eqref{eq:RS} respectively, we have for all $m\in\seq{M-1}$,
\begin{align}
    A_m &\le \sqrt{\zeta}\cdot\sqrt{A_{m + 1} + G + 2^{m + 1}\cdot (V_0 + 2R_0)} + \zeta, \label{eq:A} \\
    R_m &\le 4\hat\beta_K\sqrt{\dim_\cF}\cdot\sqrt{2R_{m+1} + A_{m + 1} + G + KH\sigmamin^2 + 2^{m + 1}\cdot (V_0 + 2R_0)} + 24\hatbeta_K^2\dim_\cF \notag \\
    &\le 4\hat\beta_K\sqrt{\dim_\cF}\cdot\sqrt{2R_{m+1} + A_{m + 1} + 2^{m + 1}\cdot (V_0 + 2R_0)} \notag \\
    &\qquad+ 4\hat\beta_K\sqrt{\dim_\cF}\cdot\sqrt{G + KH\sigmamin^2} + 24\hatbeta_K^2\dim_\cF. \label{eq:RRA}
\end{align}
Add \eqref{eq:A} to \eqref{eq:RRA}, using $\sqrt{a}+\sqrt{b}\le\sqrt{2}\cdot\sqrt{a+b}$ for any $a,b\ge0$ and assume $\zeta\le\hatbeta_K^2\dim_\cF$, we have for all $m\in\seq{M-1}$,
\begin{align}
    R_m + A_m &\le 4\hat\beta_K\sqrt{2\dim_\cF}\cdot\sqrt{2R_{m+1} + 2A_{m + 1} + 2^{m + 1}\cdot (2V_0 + 4R_0)} \notag \\
    &\qquad+ 5\hat\beta_K\sqrt{\dim_\cF}\cdot\sqrt{G + KH\sigmamin^2} + 25\hatbeta_K^2\dim_\cF \notag \\
    &= 8\hat\beta_K\sqrt{\dim_\cF}\cdot\sqrt{R_{m+1} + A_{m + 1} + 2^{m + 1}\cdot (V_0 + 2R_0)} \notag \\
    &\qquad+ 5\hat\beta_K\sqrt{\dim_\cF}\cdot\sqrt{G + KH\sigmamin^2} + 25\hatbeta_K^2\dim_\cF. \label{eq:RA}
\end{align}
And for all $m\in\seq{M}$, $R_m + \checkA_m \le 2HK$. Applying Lemma~\ref{lem:exp}, we have
\begin{equation}
\label{eq:RAV0}
    R_0 + A_0 \le 16\hatbeta_K\sqrt{\dim_\cF}\cdot\sqrt{V_0+2R_0} + 15\hatbeta_K\sqrt{\dim_\cF}\cdot\sqrt{G+KH\sigmamin^2}+331\hatbeta_K^2\dim_\cF.
\end{equation}
We have \eqref{eq:V0} holds by Lemma~\ref{lem:V0}.
Substituting the bound of $V_0$ in \eqref{eq:V0} into \eqref{eq:RAV0}, we have
\begin{align*}
    R_0 + A_0 &\le 16\hatbeta_K\sqrt{\dim_\cF}\cdot\sqrt{V^*_1K+A_0+\checkA_0+2R_0} + 15\hatbeta_K\sqrt{\dim_\cF}\cdot\sqrt{G+KH\sigmamin^2}+331\hatbeta_K^2\dim_\cF \\
    &\le 32\hatbeta_K\sqrt{\dim_\cF}\cdot\sqrt{V^*_1K+\checkA_0} + 30\hatbeta_K\sqrt{\dim_\cF}\cdot\sqrt{G+KH\sigmamin^2}+1174\hatbeta_K^2\dim_\cF,
\end{align*}
where the last inequality holds since $x\le a\sqrt{x}+b$ implies $x\le a^2+2b$ for any $x\ge0$.
\end{proof}

\subsection{Proof of Lemma~\ref{lem:regret-gap}}
\label{proof:regret-gap}
\begin{proof}[Proof of Lemma~\ref{lem:regret-gap}]
First, we decompose $V_{k,1}(s^k_1)$ and $V^{\pi^k}_1(s^k_1)$ as follows
\begin{align*}
    V_{k,1}(s^k_1) &= \sum_{h=1}^H I^k_h[V_{k,h}(s^k_h)-V_{k,h+1}(s^k_{h+1})] + \sum_{h=1}^H (1-I^k_h)[V_{k,h}(s^k_h)-V_{k,h+1}(s^k_{h+1})] \\
    &= \sum_{h=1}^H I^k_h r(s^k_h,a^k_h) + \sum_{h=1}^H I^k_h \sbr{V_{k,h}(s^k_h)-r(s^k_h,a^k_h)-[\PP V_{k,h+1}](s^k_h,a^k_h)} \\
    &\qquad+ \sum_{h=1}^H I^k_h \sbr{[\PP V_{k,h+1}](s^k_h,a^k_h)-V_{k,h+1}(s^k_{h+1})} + (1-I^k_{h_k}) V_{k,h_k}(s^k_{h_k}), \\
    V^{\pi^k}_1(s^k_1) &= \sum_{h=1}^H I^k_h[V^{\pi^k}_h(s^k_h)-V^{\pi^k}_{h+1}(s^k_{h+1})] + \sum_{h=1}^H (1-I^k_h)[V^{\pi^k}_h(s^k_h)-V^{\pi^k}_{h+1}(s^k_{h+1})] \\
    &= \sum_{h=1}^H I^k_h r(s^k_h,a^k_h) + \sum_{h=1}^H I^k_h \sbr{V^{\pi^k}_h(s^k_h)-r(s^k_h,a^k_h)-[\PP V^{\pi^k}_{h+1}](s^k_h,a^k_h)} \\
    &\qquad+ \sum_{h=1}^H I^k_h \sbr{[\PP V^{\pi^k}_{h+1}](s^k_h,a^k_h)-V^{\pi^k}_{h+1}(s^k_{h+1})} + (1-I^k_{h_k}) V^{\pi^k}_{h_k}(s^k_{h_k}) \\
    &= \sum_{h=1}^H I^k_h r(s^k_h,a^k_h) + \sum_{h=1}^H I^k_h \sbr{[\PP V^{\pi^k}_{h+1}](s^k_h,a^k_h)-V^{\pi^k}_{h+1}(s^k_{h+1})} + (1-I^k_{h_k}) V^{\pi^k}_{h_k}(s^k_{h_k}),
\end{align*}
where $h_k$ is defined in \eqref{eq:def:h_k}. Thus it follows that
\begin{align*}
    &V_{k,1}(s^k_1) - V^{\pi^k}_1(s^k_1) \\
    ={} &\sum_{h=1}^H I^k_h \sbr{V_{k,h}(s^k_h)-r(s^k_h,a^k_h)-[\PP V_{k,h+1}](s^k_h,a^k_h)} + (1-I^k_{h_k}) V_{k,h_k}(s^k_{h_k}) - (1-I^k_{h_k}) V^{\pi^k}_{h_k}(s^k_{h_k}) \\
    &\quad+ \sum_{h=1}^H I^k_h \sbr{[\PP V_{k,h+1}](s^k_h,a^k_h)-V_{k,h+1}(s^k_{h+1})} - \sum_{h=1}^H I^k_h \sbr{[\PP V^{\pi^k}_{h+1}](s^k_h,a^k_h)-V^{\pi^k}_{h+1}(s^k_{h+1})} \\
    ={} & \sum_{h=1}^H I^k_h \sbr{V_{k,h}(s^k_h)-r(s^k_h,a^k_h)-[\PP V_{k,h+1}](s^k_h,a^k_h)} + (1-I^k_{h_k}) V_{k,h_k}(s^k_{h_k}) - (1-I^k_{h_k}) V^{\pi^k}_{h_k}(s^k_{h_k}) \\
    &\quad+ \sum_{h=1}^H I^k_h \sbr{[\PP \checkV_{k,h+1}](s^k_h,a^k_h)-\checkV_{k,h+1}(s^k_{h+1})} + \sum_{h=1}^H I^k_h \sbr{[\PP \tildeV_{k,h+1}](s^k_h,a^k_h)-\tildeV_{k,h+1}(s^k_{h+1})}.
\end{align*}
Then we have
\begin{align*}
    &\sum_{k=1}^K \sbr{V_{k,1}(s^k_1)-V^{\pi^k}(s^k_1)} \\
    \overset{(a)}{\le}{} &2\sum_{k=1}^K\sum_{h=1}^H I^k_h\min\{1,\hatbeta_k\cD_\cF(z_{k,h,0};D_{k,0})\} + \sum_{k=1}^K (1-I^k_{h_k}) V_{k,h_k}(s^k_{h_k}) \\
    &\quad+ \sum_{k=1}^K\sum_{h=1}^H I^k_h \sbr{[\PP \checkV_{k,h+1}](s^k_h,a^k_h)-\checkV_{k,h+1}(s^k_{h+1})} + \sum_{k=1}^K\sum_{h=1}^H I^k_h \sbr{[\PP \tildeV_{k,h+1}](s^k_h,a^k_h)-\tildeV_{k,h+1}(s^k_{h+1})} \\
    \overset{(b)}{\le}{} &2R_0 + \checkA_0 + \tildeA_0 + G,
\end{align*}
where $(a)$ is due to Lemma~\ref{lem:event-bonus} and $(b)$ holds since $|V_{k, h}(s_{h}^k)|\le1$ and $I^k_{h}$ is monotonically decreasing.
\end{proof}

\ifarxiv
\section{Auxiliary Lemmas}
\else
\section{AUXILLIARY LEMMAS}
\fi
\label{proof:auxiliary}

\begin{lemma}[Variance-aware and range-aware Freedman's inequality, Corollary~2 in \citet{agarwal2023vo}]
\label{lem:freedman-variant}
Let $M\ge m>0,V\ge v>0$ be fixed constants, and $\{X_s\}_{s\in[t]}$ be a stochastic process adapted to the filtration $\{\cG_s\}_{s\in[t]}$, such that $X_s$ is $\cG_{s}$-measurable. Suppose $\EE[X_s|\cG_{s-1}]=0$, $|X_s|\le M$ and $\sum_{s=1}^t\EE[X_s^2|\cG_{s-1}]\le V^2$ almost surely. Then for any $\delta>0$, with probability at least $1 - (\log(V^2/v^2)+2)(\log(M/m)+2)\delta$, we have
\[
    \sum_{s = 1}^t X_s \le \sqrt{2\rbr{2\sum_{s=1}^t\EE[X_s^2|\cG_{s-1}]+v^2}\log\frac{1}{\delta}}+\frac{2}{3}\rbr{2\max_{s\in[t]}|X_s|+m}\log\frac{1}{\delta}. 
\]
\end{lemma}

\begin{lemma}[Variance-aware Freedman's inequality]
\label{lem:freedman-variance}
Let $M>0$ be fixed constants, and $\{X_s\}_{s\in[t]}$ be a stochastic process adapted to the filtration $\{\cG_s\}_{s\in[t]}$, such that $X_s$ is $\cG_{s}$-measurable. Suppose $\EE[X_s|\cG_{s-1}]=0$ and $|X_s|\le M$ almost surely. Then for any $\delta>0$, with probability at least $1-2(\log t + 2)\delta$, we have
\begin{align*}
    \sum_{s=1}^t X_s\le2\sqrt{\sum_{s=1}^t\EE[X_s^2|\cG_{s-1}]\log\frac{1}{\delta}}+4M\log\frac{1}{\delta}.
\end{align*}
\end{lemma}
\begin{proof}
The result follows by applying Lemma~\ref{lem:freedman-variant} with $V^2=M^2 t,m=v=M$.
\end{proof}

\begin{lemma}[Elliptical Potential Lemma, Lemma 11 in \citet{abbasi2011improved}]
\label{lem:elliptical}
Let $\{\bx_t\}_{t\in[T]} \subset \RR^d$ and assume $\|\bx_t\|_2 \le L$ for all $t\in[T]$. Set $\bSigma_t = \sum_{s=1}^t \bx_t \bx_t^\top + \lambda \bI$. Then it follows that
\[
    \sum_{t=1}^T \min \cbr{1, \|\bx_t\|^2_{\bSigma_{t-1}^{-1}}} \le 2d \log\rbr{1 + \frac{T L^2}{d\lambda}}.
\]
\end{lemma}

\begin{lemma}
\label{lem:cD-relation}
Let $\cD_\cF$ in Definition~\ref{def:general-eluder-RL}, for any $t>t_0\ge1$, we have
\[
    \cD_\cF^2(X_t;X_{[t_0]},\sigma_{[t_0]}) \le \exp\cbr{\sum_{s=t_0+1}^{t-1} \frac{1}{\sigma_s^2}\cD_\cF^2(X_s;X_{[s-1]},\sigma_{[s-1]})} \cD_\cF^2(X_t;X_{[t-1]},\sigma_{[t-1]}).
\]
\end{lemma}
\begin{proof}
Note
\begin{align*}
    &\cD_\cF^2(X_t;X_{[t-1]},\sigma_{[t-1]})\\
    ={} &\sup_{f_1,f_2\in\cF}\frac{(f_1(X_t)-f_2(X_t))^2}{\sum_{s=1}^{t-1} \frac{1}{\sigma_s^2}(f_1(X_s)-f_2(X_s))^2+\lambda}\\
    \ge{} &\sup_{f_1,f_2\in\cF}\frac{(f_1(X_t)-f_2(X_t))^2}{(1+\frac{1}{\sigma_{t-1}^2}\cD_\cF^2(X_{t-1};X_{[t-2]},\sigma_{[t-2]}))\sum_{s=1}^{t-2} \frac{1}{\sigma_s^2}(f_1(X_s)-f_2(X_s))^2+\lambda}\\
    \ge{} &\frac{1}{1+\frac{1}{\sigma_{t-1}^2}\cD_\cF^2(X_{t-1};X_{[t-2]},\sigma_{[t-2]})} \sup_{f_1,f_2\in\cF}\frac{(f_1(X_t)-f_2(X_t))^2}{\sum_{s=1}^{t-2} \frac{1}{\sigma_s^2}(f_1(X_s)-f_2(X_s))^2+\lambda}\\
    ={} &\frac{1}{1+\frac{1}{\sigma_{t-1}^2}\cD_\cF^2(X_{t-1};X_{[t-2]},\sigma_{[t-2]})} \cD_\cF^2(X_t;X_{[t-2]},\sigma_{[t-2]})\ge\dots\\
    \ge{} &\frac{1}{\prod_{s=t_0+1}^{t-1} (1+\frac{1}{\sigma_{s}^2}\cD_\cF^2(X_s;X_{[s-1]},\sigma_{[s-1]}))} \cD_\cF^2(X_t;X_{[t_0]},\sigma_{[t_0]}),
\end{align*}
where the first inequality holds due to
\begin{align*}
    &\sum_{s=1}^{t-1} \frac{1}{\sigma_s^2}(f_1(X_s)-f_2(X_s))^2\\
    ={}& \sum_{s=1}^{t-2} \frac{1}{\sigma_s^2}(f_1(X_s)-f_2(X_s))^2 + \frac{1}{\sigma_{t-1}^2}(f_1(X_{t-1})-f_2(X_{t-1}))^2\\
    \le{}& \sum_{s=1}^{t-2} \frac{1}{\sigma_s^2}(f_1(X_s)-f_2(X_s))^2 + \frac{1}{\sigma_{t-1}^2}\cD_\cF^2(X_{t-1};X_{[t-2]},\sigma_{[t-2]}) \sum_{s=1}^{t-2} \frac{1}{\sigma_s^2}(f_1(X_s)-f_2(X_s))^2\\
    ={}& \rbr{1+\frac{1}{\sigma_{t-1}^2}\cD_\cF^2(X_{t-1};X_{[t-2]},\sigma_{[t-2]})}\sum_{s=1}^{t-2} \frac{1}{\sigma_s^2}(f_1(X_s)-f_2(X_s))^2.
\end{align*}
Thus, we have
\begin{align*}
    \cD_\cF^2(X_t;X_{[t_0]},\sigma_{[t_0]}) &\le \prod_{s=t_0+1}^{t-1}\rbr{1+\frac{1}{\sigma_{s}^2}\cD_\cF^2(X_s;X_{[s-1]},\sigma_{[s-1]})} \cD_\cF^2(X_t;X_{[t-1]},\sigma_{[t-1]})\\
    &\le \exp\cbr{\sum_{s=t_0+1}^{t-1} \frac{1}{\sigma_{s}^2}\cD_\cF^2(X_s;X_{[s-1]},\sigma_{[s-1]})} \cD_\cF^2(X_t;X_{[t-1]},\sigma_{[t-1]}),
\end{align*}
where the second inequality holds due to the inequality $1+x\le \exp\{x\}$.
\end{proof}

\begin{lemma}
\label{lem:sum-bonus}
Let $\{\sigma_t, \beta_t\}_{t \ge 1}$ be a sequence of non-negative numbers, $\sigmamin, \gamma,\lambda>0$, $\{X_t\}_{t\ge1}\subset\cX$ and $\{\bar\sigma_k\}_{k \geq 1}$ be recursively defined:
\[
    \barsigma_t^2 = \max\{\sigma_t^2, \sigmamin^2, \gamma^2\cD_\cF(X_t;X_{[t-1]},\sigma_{[t-1]})\}.
\]
Then we have
\[
    \sum_{t=1}^T \min\{1,\beta_t\cD_\cF(X_t;X_{[t-1]},\sigma_{[t-1]})\} \le \dim_\cF + \max_{t\in[T]}\beta_t\cdot\gamma^2\dim_\cF + \sqrt{\dim_\cF}\cdot\sqrt{\sum_{t=1}^T\beta_t^2(\sigma_t^2+\sigmamin^2)},
\]
where $\cD_\cF$ and $\dim_\cF=\dim_\cF(\sigmamin,T)$ are in Definition~\ref{def:general-eluder-RL}.
\end{lemma}
\begin{proof}
We decompose $[T]$ as the union of three disjoint sets $\cJ_1,\cJ_2,\cJ_3$:
\begin{align*}
    \cJ_1 &= \cbr{t\in[T]\Big|\frac{1}{\barsigma_t}\cD_\cF(X_t;X_{[t-1]},\sigma_{[t-1]})>1},\\
    \cJ_2 &= \cbr{t\in[T]\Big|\frac{1}{\barsigma_t}\cD_\cF(X_t;X_{[t-1]},\sigma_{[t-1]})\le1,\barsigma_t\in\{\sigma_t,\sigmamin\}},\\
    \cJ_3 &= \cbr{t\in[T]\Big|\frac{1}{\barsigma_t}\cD_\cF(X_t;X_{[t-1]},\sigma_{[t-1]})\le1,\barsigma_t=\gamma\sqrt{\cD_\cF(X_t;X_{[t-1]},\sigma_{[t-1]})}}.
\end{align*}
For the summation over $\cJ_1$, we have
\[
    \sum_{t\in\cJ_1}\min\{1,\beta_t\cD_\cF(X_t;X_{[t-1]},\sigma_{[t-1]})\}
    \le |\cJ_1|
    \le \sum_{t\in\cJ_1}\min\cbr{1,\frac{1}{\barsigma_t^2}\cD_\cF^2(X_t;X_{[t-1]},\sigma_{[t-1]})}
    \le \dim_\cF.
\]
Next, for the summation over $\cJ_2$, we have
\begin{align*}
    &\sum_{t\in\cJ_2}\min\{1,\beta_t\cD_\cF(X_t;X_{[t-1]},\sigma_{[t-1]})\}\\
    \le{} &\sum_{t\in\cJ_2}\beta_t\barsigma_t\cdot\frac{1}{\barsigma_t}\cD_\cF(X_t;X_{[t-1]},\sigma_{[t-1]})\}\\
    \le{} &\sum_{t=1}^T\beta_t\max\{\sigma_t,\sigmamin\}\frac{1}{\barsigma_t}\cD_\cF(X_t;X_{[t-1]},\sigma_{[t-1]})\}\\
    \overset{(a)}{\le}{} &\sqrt{\sum_{t=1}^T\beta_t^2(\sigma_t^2+\sigmamin^2)}\cdot\sqrt{\sum_{t=1}^T\min\cbr{1,\frac{1}{\barsigma_t^2}\cD_\cF^2(X_t;X_{[t-1]},\sigma_{[t-1]})}}\\
    \le{} &\sqrt{\dim_\cF}\cdot\sqrt{\sum_{t=1}^T\beta_t^2(\sigma_t^2+\sigmamin^2)},
\end{align*}
where $(a)$ holds due to Cauchy-Schwartz inequality and $\frac{1}{\barsigma_t}\cD_\cF(X_t;X_{[t-1]},\sigma_{[t-1]})\le1$.
Then, for the summation over $\cJ_3$, we have
\begin{align*}
    &\sum_{t\in\cJ_3}\min\{1,\beta_t\cD_\cF(X_t;X_{[t-1]},\sigma_{[t-1]})\}\\
    \le{} &\sum_{t\in\cJ_3}\beta_t\barsigma_t\cdot\frac{1}{\barsigma_t}\cD_\cF(X_t;X_{[t-1]},\sigma_{[t-1]})\\
    \overset{(a)}{\le}{} &\max_{t\in[T]}\beta_t\cdot\gamma^2\sum_{t=1}^T\min\cbr{1,\frac{1}{\barsigma_t^2}\cD_\cF^2(X_t;X_{[t-1]},\sigma_{[t-1]})}\\
    \le{} &\max_{t\in[T]}\beta_t\cdot\gamma^2\dim_\cF,
\end{align*}
where $(a)$ holds due to $\barsigma_t=\gamma^2\frac{1}{\barsigma_t}\cD_\cF(X_t;X_{[t-1]},\sigma_{[t-1]})$ and $\frac{1}{\barsigma_t}\cD_\cF(X_t;X_{[t-1]},\sigma_{[t-1]})\le1$.
Finally, putting pieces together finishes the proof.
\end{proof}

\begin{lemma}[Modified from Lemma~2 in \citealt{zhang2021reinforcement}]
\label{lem:exp}
Let $\lambda_1, \lambda_2, \lambda_4  > 0$, $\lambda_3 \ge 1$ and $i' = \lceil \log_2 \lambda_1 \rceil$. Let $a_0, a_1, a_2, \dots, a_{i'}$ be non-negative reals such that $a_i \le \lambda_1$ for any $0 \le i \le i'$, and $a_i \le \lambda_2 \sqrt{a_{i + 1} + 2^{i + 1} \cdot \lambda_3} + \lambda_4$ for any $0 \le i < i'$. Then we have
\begin{align*}
    a_0 &\le \max\cbr{\rbr{\lambda_2 + \sqrt{\lambda_2^2 + \lambda_4}}^2, \lambda_2 \sqrt{4 \lambda_3} + \lambda_4} \le \lambda_2 \sqrt{4 \lambda_3} + 4 \lambda_2^2 + 3 \lambda_4, \\
    a_1 &\le \max\cbr{\rbr{\lambda_2 + \sqrt{\lambda_2^2 + \lambda_4}}^2, \lambda_2 \sqrt{8 \lambda_3} + \lambda_4} \le \lambda_2 \sqrt{8 \lambda_3} + 4 \lambda_2^2 + 3 \lambda_4.
\end{align*}
\end{lemma}

\end{document}